\documentclass{article}

\usepackage[preprint]{icml2026} 

\usepackage{microtype}
\usepackage{graphicx}
\usepackage{subcaption}
\usepackage{booktabs} 
\usepackage{amsmath}
\usepackage{amssymb}
\usepackage{mathtools}
\usepackage{enumerate}
\usepackage{enumitem}
\usepackage{amsthm}
\usepackage{tabularx}
\usepackage{multirow}
\usepackage{array}
\usepackage{makecell}
\usepackage[table]{xcolor}
\definecolor{ICMLBlue}{RGB}{25, 88, 166}
\definecolor{ICMLPurple}{RGB}{150, 0, 80}
\usepackage[colorlinks,linkcolor=ICMLPurple,citecolor=ICMLBlue,urlcolor=black]{hyperref}
\usepackage{bm}

\usepackage[capitalize,noabbrev]{cleveref}

\theoremstyle{plain}
\newtheorem{theorem}{Theorem}[section]

\newtheorem{lemma}[theorem]{Lemma}

\theoremstyle{definition}
\newtheorem{definition}[theorem]{Definition}

\theoremstyle{remark}

\usepackage[textsize=tiny]{todonotes}

\renewcommand{\paragraph}[1]{\noindent\textbf{#1}~}

\begin{document}

\icmltitlerunning{Weight-Informed Self-Explaining Clustering for Mixed-Type Tabular Data}

\twocolumn[
  \icmltitle{Weight-Informed Self-Explaining Clustering for Mixed-Type Tabular Data}




  \begin{icmlauthorlist}
    \icmlauthor{Lehao Li}{nus}
    \icmlauthor{Qiang Huang}{hitsz}
    \icmlauthor{Yihao Ang}{nus}
    \icmlauthor{Bryan Kian Hsiang Low}{nus}
    \icmlauthor{Anthony K. H. Tung}{nus}
    \icmlauthor{Xiaokui Xiao}{nus}
  \end{icmlauthorlist}

  \icmlaffiliation{nus}{School of Computing, National University of Singapore, Singapore}
  \icmlaffiliation{hitsz}{School of Intelligence Science and Engineering, Harbin Institute of Technology (Shenzhen), China}

  \icmlcorrespondingauthor{Qiang Huang}{huangqiang@hit.edu.cn}

  \icmlkeywords{Tabular Data Clustering, Interpretability, Feature Representation}

  \vskip 0.3in
]



\printAffiliationsAndNotice{}  

\begin{abstract}
Clustering mixed-type tabular data is fundamental for exploratory analysis, yet remains challenging due to misaligned numerical-categorical representations, uneven and context-dependent feature relevance, and disconnected and post-hoc explanation from the clustering process. 
We propose \textbf{WISE}, a \textbf{W}eight-\textbf{I}nformed \textbf{S}elf-\textbf{E}xplaining framework that unifies representation, feature weighting, clustering, and interpretation in a fully unsupervised and transparent pipeline. 
\textbf{WISE} introduces Binary Encoding with Padding (BEP) to align heterogeneous features in a unified sparse space, a Leave-One-Feature-Out (LOFO) strategy to sense multiple high-quality and diverse feature-weighting views, and a two-stage weight-aware clustering procedure to aggregate alternative semantic partitions. 
To ensure intrinsic interpretability, we further develop Discriminative FreqItems (DFI), which yields feature-level explanations that are consistent from instances to clusters with an additive decomposition guarantee. 
Extensive experiments on six real-world datasets demonstrate that \textbf{WISE} consistently outperforms classical and neural baselines in clustering quality while remaining efficient, and produces faithful, human-interpretable explanations grounded in the same primitives that drive clustering.
\end{abstract}


\section{Introduction}
\label{sec:intro}

Clustering is a cornerstone of exploratory data analysis, enabling structure discovery, cohort formation, and hypothesis generation in the absence of labels \cite{jain2010data, xu2005survey, huang2023kfreqitems}.
In modern applications, many high-value datasets are mixed-type tabular tables, combining both numerical and categorical attributes across domains such as demography, privacy analytics, electronic health records, intrusion detection, behavioral marketing, and financial risk assessment \cite{huang1997kprototypes, cad, ang2024eads, gorishniy2021revisiting, shwartz2022tabular}.
Effective clustering in this setting is particularly valuable: it reveals latent subpopulations, guides feature engineering, and produces cohorts that downstream tasks, ranging from supervised modeling to causal analysis and auditing, can directly leverage

Despite decades of progress, clustering mixed-type tabular data remains challenging due to three tightly coupled issues.


\paragraph{(1) Representation Misalignment}
A common strategy is to one-hot encode categorical attributes, concatenate them with normalized numerical features, and apply mixed-type distances (e.g., Gower \citep{gower1971gower-distance}) or prototype-based objectives such as $k$-Prototypes \cite{huang1997kprototypes}.
However, this coupling is inherently brittle: one-hot encoding expands categorical attributes into high-dimensional, sparse vectors, while numerical features remain low-dimensional and ordered, forcing a single global scaling heuristic to reconcile incompatible geometries \cite{jain2010data}.
An alternative line of work embeds heterogeneous features into dense latent spaces and performs clustering via $k$-Means or deep objectives \cite{guo2016entityembeddings, somepalli2021saint, huang2020tabtransformer, xie2016DEC}.
While effective for predictive modeling \cite{arik2021tabnet, gorishniy2021revisiting, rauf2025TableDC, svirsky2024IDC}, these approaches shift similarity to latent dimensions whose semantics are \emph{opaque}, difficult to align across feature types, and poorly suited for feature-level weighting or explanation in the original table \cite{shwartz2022tabular}.

\paragraph{(2) Uneven and Context-Dependent Feature Relevance}
Real-world tabular data rarely exhibits uniform feature importance: a small subset of attributes often dominates similarity, while many others are weak, redundant, or relevant only for specific subpopulations.
Yet most clustering objectives implicitly weight all dimensions equally once a representation is fixed, diluting informative features and obscuring alternative but valid clusterings driven by different feature subsets \cite{agrawal1998clique, aggarwal1999proclus, parsons2004subspace}.
Metric learning and constraint-based clustering can adapt distances, but they typically rely on side information and do not yield multiple, diverse feature weightings \cite{xing2002metric, bilenko2004constraints}.
Deep tabular embeddings face a related limitation, as latent dimensions are usually aggregated isotropically, making it difficult to audit which original features actually drive separation \cite{gorishniy2021revisiting, grinsztajn2022tree, shwartz2022tabular}.

\paragraph{(3) Disconnected and Post-Hoc Explanations}
Although substantial work explains black-box predictions \cite{ribeiro2016LIME, lundberg2017shap, lundberg2020treeshap}, and advances interpretable clustering \cite{peng2022tell, lawless2022multipolytope, moshkovitz2020explainable}, explanations for mixed-type tabular clustering are often \emph{post hoc}, disconnected from the similarity computations, or defined in latent spaces.
As a result, cluster-level summaries do not reliably decompose into instance-level rationales in the original feature space \cite{rauf2025TableDC, svirsky2024IDC, spagnol2024counterfactualclustering}.
This gap matters in practice, where analysts require both \emph{global narratives} (what defines a cohort) and \emph{local justifications} (why a record belongs to it), derived from the same primitives that produce the clustering.

To address these challenges, we introduce \textbf{WISE} (\textbf{W}eight-\textbf{I}nformed \textbf{S}elf-\textbf{E}xplaining), a unified framework for clustering mixed-type tabular data, guided by three principles:
(1) \textbf{Granularity alignment}, mapping numerical and categorical attributes into a shared, type-respecting discrete space;
(2) \textbf{Feature weighting as a first-class object}, sensing multiple, diverse feature-weightings directly from data to guide clustering under alternative semantic emphases; and 
(3) \textbf{Intrinsic interpretability}, embedding explanations into the clustering mechanism itself with instance-to-cluster consistency and an additive, Shapley-style decomposition \cite{shapley1953value, lundberg2017shap}.

\paragraph{Contributions}
Our main contributions are summarized as:
\begin{itemize}[nolistsep]
  \vspace{-0.5em}
  \item \textbf{Granularity-Aligned Mixed-Type Representation:} 
  We propose \emph{Binary Encoding with Padding (BEP)}, a unified sparse encoding that preserves categorical semantics and numerical order without a global numeric-categorical balance parameter.
  
  \item \textbf{Diverse Feature-Weight Sensing:} 
  We design a \emph{Leave-One-Feature-Out (LOFO)} Random Forest framework that extracts multiple high-quality, diverse feature-weight vectors from reconstruction-based attributions.
  
  \item \textbf{Intrinsic and Consistent Interpretability:} 
  We develop \emph{Discriminative FreqItems (DFI)}, which ties explanations directly to clustering primitives and yields coherent feature-level attributions across instance and cluster levels with a theoretical guarantee.
\end{itemize}
\vspace{-0.5em}

\section{Related Work}
\label{sec:related_work}

\begin{figure*}[t]
  \centering
  \includegraphics[width=0.99\textwidth]{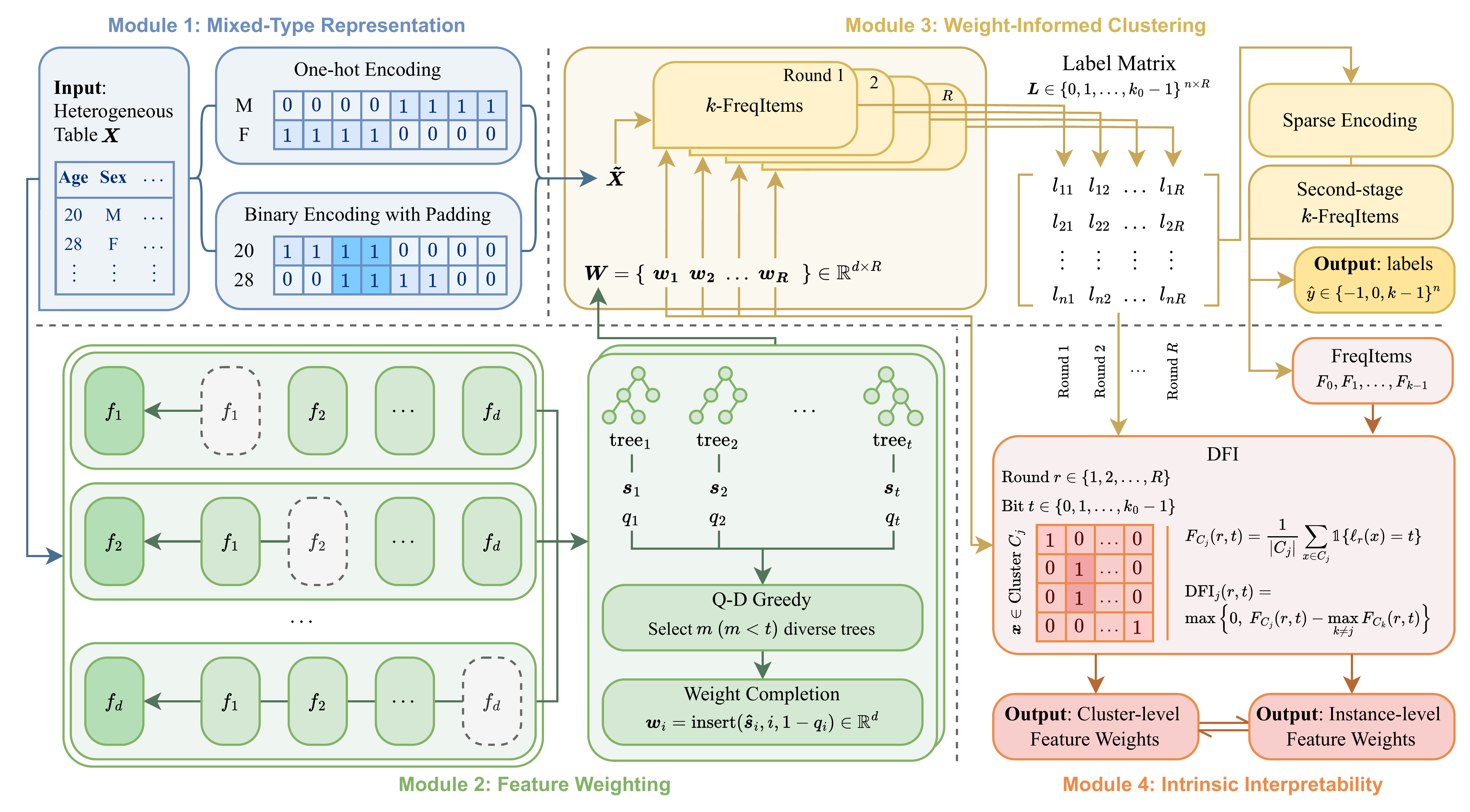}
  \vspace{-0.5em}
  \caption{Overview of \textbf{WISE}. It consists of four modules: \textbf{Module~1} converts mixed-type tabular data into a unified representation using Binary Encoding with Padding (BEP); \textbf{Module~2} senses and selects diverse feature-weight vectors via a Leave-One-Feature-Out (LOFO) strategy; \textbf{Module~3} aggregates multiple weighted views through a two-stage, weight-informed clustering procedure; Finally, \textbf{Module~4} produces intrinsic and consistent explanations at both the cluster and instance levels using Discriminative FreqItems (DFI).} 
  \label{fig:pipeline_WISE}
  \vspace{-0.5em}
\end{figure*}

\paragraph{Mixed-Type Tabular Clustering}
Clustering mixed-type tabular data has a long history, but most methods prioritize partition quality over interpretability.
Classical algorithms such as \textbf{$k$-Prototypes} \citep{huang1997kprototypes} extend $k$-Means \citep{macqueen1967some, lloyd1982least} using hybrid numerical-categorical dissimilarities, while general frameworks use mixed-type similarities like Gower distance \citep{gower1971gower-distance} for clustering (e.g., DBSCAN \citep{ester1996dbscan}).
Recent learning-based methods embed heterogeneous features into latent spaces before clustering; for example, \textbf{TableDC} \citep{rauf2025TableDC} employs autoencoder-based architectures.
Although effective for separation, these methods define clusters via latent embeddings or prototypes, offering limited insight into which original features distinguish clusters.

\paragraph{Interpretable Clustering}
Motivated by this limitation, interpretable clustering has received growing attention.
Intrinsically transparent models include \textbf{TELL} \citep{peng2022tell}, which reformulates $k$-Means as a white-box neural network, and \textbf{IDC} \citep{svirsky2024IDC}, which jointly learns cluster assignments and sparse feature gates.
Declarative approaches further impose explicit structures such as polytope regions \citep{lawless2022multipolytope} or axis-aligned rules \citep{moshkovitz2020explainable}, while post-hoc techniques such as counterfactual explanations explain clustering outcomes without modifying the algorithm \citep{spagnol2024counterfactualclustering}. 
However, most existing methods assume \emph{homogeneous} (typically numeric) features or impose restrictive structures, decoupling interpretability from the clustering objective and limiting applicability to mixed-type tables.

\paragraph{Feature Representation and Weighting}
Feature representation and weighting pose additional challenges.
Common encodings, including learned categorical embeddings \cite{rendle2010FM, cheng2016wide, guo2016entityembeddings, guo2017deepfm, huang2020tabtransformer}
and target encoding \cite{micci2001target-encoding}, often rely on supervision or yield dense, weakly interpretable representations.
Attention-based tabular models (e.g., \textbf{SAINT} \citep{somepalli2021saint}) are optimized for supervised prediction rather than unsupervised clustering.
Sparse methods such as \textbf{$k$-FreqItems} \citep{huang2023kfreqitems} improve interpretability for binary data but do not address heterogeneous feature weighting.
Meanwhile, feature relevance techniques from supervised learning \cite{lundberg2017shap, arik2021tabnet} do not readily transfer to clustering due to the absence of labels.
\textbf{WISE} fills these gaps by integrating representation, feature weighting, clustering, and explanation into a unified, fully unsupervised framework for mixed-type tabular data.
By operating directly on semantically meaningful feature-level signals, it yields human-interpretable cluster descriptions without opaque latent spaces or post-hoc explanations.

\section{The WISE Framework}
\label{sec:method}

We introduce \textbf{WISE}, a weight-informed self-explaining framework for clustering mixed-type tabular data, which unifies representation, feature weighting, clustering, and explanation in a fully interpretable manner (Figure \ref{fig:pipeline_WISE}).

\subsection{Module 1: Mixed-Type Representation}
\label{sec:method:BEP}

\paragraph{Challenges}
Clustering mixed-type tabular data requires representations that are both unified and feature-interpretable. 
Existing solutions largely follow two paradigms: distance-based approaches \cite{gower1971gower-distance} reconcile numerical and categorical features via heuristic weighting across incompatible geometries, while neural tabular encoders \cite{lundberg2017shap, arik2021tabnet, rauf2025TableDC} embed all features into dense latent spaces, enabling standard distances but obscuring feature semantics and interpretability.

\paragraph{Binary Encoding with Padding (BEP)}
To overcome these limitations, we propose a BEP scheme, which maps both numerical and categorical features into a \emph{single sparse binary space}.
This unified representation supports feature-level weighting and similarity computation via weighted Jaccard distance, while preserving a direct and interpretable correspondence to the original tabular features.

Concretely, each feature is encoded as a fixed-length binary vector of $2B$ bits. 
Categorical features are represented using standard one-hot encoding. 
For numeric features, we propose a positional padding scheme: each value is mapped to a contiguous block of $B$ ones, whose offset reflects its relative magnitude within the column. 
Smaller values align to the left, larger values to the right, with intermediate values placed proportionally in between. 

\begin{figure}[t]
  \centering
  \includegraphics[width=0.8\columnwidth]{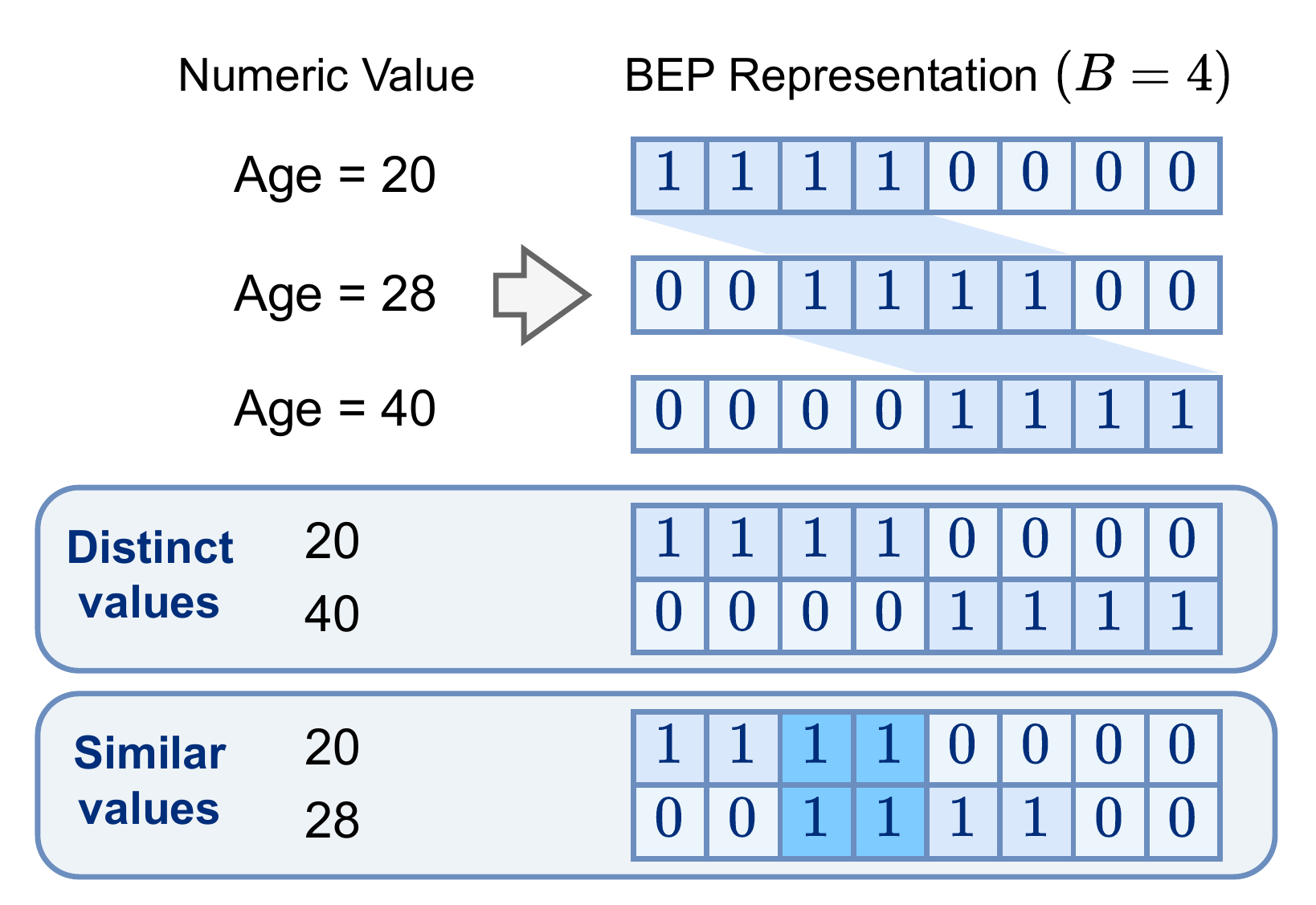}
  \vspace{-0.5em}
  \caption{An illustrative example of the BEP scheme.}
  \label{fig:example_BEP}
  \vspace{-1.5em}
\end{figure}

\paragraph{Remarks}
Figure~\ref{fig:example_BEP} provides an illustrative example for a numerical feature. 
The smallest value (e.g., 20) has no overlap with the largest value (e.g., 40), but shares 2 bits with a nearby value (e.g., 28).
Consequently, numerically distant values exhibit minimal overlap and large Jaccard distances, while similar values share more bits and have smaller distances.
This design allows Jaccard distance in the BEP space to faithfully reflect original numeric similarity, while maintaining a unified, sparse, and interpretable representation across all feature types.

\begin{lemma}[\textbf{Quantization Error Bounds for BEP Jaccard Distance}] 
\label{lem:bep-quant-bracket}
  Let $x,y\in[0,1]$ and define $s(x) = \mathrm{round}(Bx) \in \{0,\cdots,B\}$. 
  Let $\bm{\phi}(x)\in\{0,1\}^{2B}$ denote the BEP encoding with a contiguous block of $B$ ones starting at position $s(x)$.
  Let $t:=|x-y|$, and define $t_- := \max\{0,t-\tfrac{1}{B}\}$, $t_+ := \min\{1,t+\tfrac{1}{B}\}$. 
  Then the Jaccard distance between BEP codes satisfies:
  \begin{displaymath}
    \textstyle \frac{2t_-}{1+t_-} \leq d_J(\bm{\phi}(x),\bm{\phi}(y)) \leq \frac{2t_+}{1+t_+}.
  \end{displaymath}
\end{lemma}

\begin{proof}
  \vspace{-1.0em}
  By nearest rounding, we have $|\tfrac{s(x)}{B} - x| \leq \tfrac{1}{2B}$ and $|\tfrac{s(y)}{B} - y| \leq \tfrac{1}{2B}$.
  Applying the triangle inequality yields: 
  \begin{displaymath}
    {\bigl|\, |\tfrac{s(x)-s(y)}{B}|-|x-y|\,\bigr|} \leq \tfrac{1}{B}. 
  \end{displaymath}
  Let $\tau:=|\tfrac{s(x)-s(y)}{B}|$. Then $|\tau-t|\le \frac{1}{B}$, implying 
  \begin{displaymath} 
    t_- \leq \tau \leq t_+.
  \end{displaymath}
  For BEP, the overlap and union depend only on the relative shift $\tau$, yielding the closed-form Jaccard distance:
  \begin{displaymath}
    d_J(\bm{\phi}(x), \bm{\phi}(y)) = \tfrac{2\tau}{1+\tau},
  \end{displaymath}
  which is monotone increasing in $\tau$. Substituting the bounds on $\tau$ completes the proof.
\end{proof}

\subsection{Module 2: Feature Weighting}
\label{sec:method:weights}

Beyond representation, effective clustering of mixed-type tabular data hinges on \emph{feature weighting}. 
Unlike embedding dimensions, tabular features correspond to heterogeneous real-world measurements whose relevance is often uneven and context-dependent. 
Uniformly weighting all features can dilute informative attributes and obscure latent cluster structure. 
Yet, clustering is inherently unsupervised, making it difficult to estimate feature importance without labels.

To address this, we reformulate unsupervised feature weighting as a set of supervised \emph{dependency estimation} tasks via a \textbf{Leave-One-Feature-Out (LOFO)} strategy. 
This design is conceptually inspired by masked reconstruction objectives in language \cite{devlin2019bert, clarkelectra} and vision \cite{pathak2016context, he2022masked, xie2022simmim}, which uncover structured dependencies by predicting missing components from observed context. 
For each feature $x_j$, we predict it from the remaining features $X_{\setminus j}$:
\begin{displaymath}
  f_j: X_{\setminus j}\rightarrow x_j, \quad j \in \{1,\cdots,d\}.
\end{displaymath}
This yields $d$ supervised tasks, each quantifying how strongly a feature is supported by the rest of the table.

\paragraph{Sensing via LOFO and TreeSHAP}
We instantiate each predictor $f_j$ using a Random Forest (RF), chosen for its robustness to heterogeneous data and ability to capture non-linear interactions \cite{breiman2001random, au2018random}. 
RF classifiers are used for categorical targets, and RF regressors for numerical targets.
Crucially, RFs admit efficient and principled feature attributions via \textbf{TreeSHAP} \cite{lundberg2020treeshap}, which provides Shapley-value-based credit assignment; more details are provided in Appendix \ref{app:impl:treeshap}.

For each target feature $x_j$, we train an RF on a subsampled training set and evaluate it on held-out data. 
For each tree $u$ in the ensemble, we calculate a global attribution vector $\bm{s}_{j,u} \in \mathbb{R}_{\geq 0}^{d-1}$ over the remaining features by aggregating absolute TreeSHAP values. These vectors encode how different subsets of features contribute to explaining $x_j$.

\paragraph{Selection via Quality-Diversity Greedy Search}
Trees in an RF are stochastic and heterogeneous: some are accurate but redundant (provide near-duplicate dependency patterns), while others are diverse but unreliable. 
Since our goal is to produce multiple \emph{distinct yet trustworthy} weighting views for each feature $x_j$, we explicitly balance quality and diversity when selecting trees.
For tree $u$ predicting $x_j$, we define $q_{j,u}\in[0,1]$ as the normalized validation performance. We use the coefficient of determination $R^2$ for regression tasks and accuracy for classification tasks. 
Diversity between trees is measured by the dissimilarity of their attribution vectors, quantified via cosine distance.
Given $T$ trees, we select a subset $\mathcal{U}_j \subseteq \{1,\cdots,T\}$ of size $m$ by maximizing:
\begin{displaymath}
  \textstyle J_j(\mathcal{S})= \frac{\lambda }{|\mathcal{S}|} \sum_{u \in \mathcal{S}} q_{j,u} + \frac{(1-\lambda) \cdot 2}{|\mathcal{S}|(|\mathcal{S}|-1)} \sum_{u<v \in \mathcal{S}} (1-\cos(u,v)),
\end{displaymath}
where $\lambda\in[0,1]$ controls the trade-off between quality and diversity, and $\mathcal{S}$ denotes a candidate subset of trees.

We optimize $J_j$ via greedy selection \cite{gollapudi2009result-diversification, borodin2017maxsum, nemhauser1978submodular}, iteratively adding the tree with the largest marginal gain. 
This yields a compact, diverse set $\mathcal{U}_j$, in which each tree corresponds to a distinct dependency view of feature $x_j$.
Implementation details are provided in Appendix \ref{app:impl:qd-greedy}.

\paragraph{Completion into Full Feature-Weight Vectors}
Each selected tree $u \in \mathcal{U}_j$ is then converted into a probability-like feature-weight vector $\bm{w}_{j,u} \in \Delta^{d}$, where
$\Delta^{d}=\{\bm{w}\in \mathbb{R}^{d}:\ \bm{w}\ge 0,\ \|\bm{w}\|_{1}=1\}$; this enforces non-negativity and unit $\ell_{1}$-norm.
We normalize its attribution vector over the observed features, scale it by the validation quality $q_{j,u}$, and assign the remaining mass to the held-out feature $x_j$:
\begin{displaymath}
  \bm{w}_{j,u}(k) = 
  \begin{cases}
    q_{j,u} \cdot \tfrac{\bm{s}_{j,u}}{\sum_{k \neq j} \bm{s}_{j,u}(k)}, & k \neq j,\\
    1 - q_{j,u}, & k = j.
  \end{cases}
\end{displaymath}
This construction is deliberately conservative: when $x_j$ is highly predictable, weight mass shifts toward its supporting features; when $x_j$ is weakly predictable, mass remains on $x_j$ itself. 
The result is a set of diverse, interpretable feature-weight vectors that faithfully reflect feature dependencies in the data and directly guide the subsequent clustering stage.

\subsection{Module 3: Weight-Informed Clustering}
\label{sec:method:clustering}

Mixed-type tabular data often admits multiple plausible clusterings, as different subsets of features can dominate similarity under different semantics. 
\textbf{WISE} therefore treats each feature-weight vector from Module~2 as a distinct view of the BEP-encoded data and aggregates these views through a two-stage weight-informed clustering procedure.

\paragraph{Stage I: Repeated Weighted $k$-FreqItems}
Let $\bm{X}^{\text{BEP}}\in\{0,1\}^{n\times p}$ be the BEP encoding.
Module~2 produces $R$ feature-weight vectors $\{\bm{w}^{(r)}\}_{r=1}^{R}$, where each $\bm{w}^{(r)} = [w^{(r)}_1, \cdots, w^{(r)}_d] \in \Delta^{d}$ assigns non-negative weights to the $d$ original features. 
Since each feature corresponds to a fixed group of BEP bits, we lift feature weights to the bit level by assigning all bits derived from feature $j$ the same weight $w^{(r)}_j$.
Using these bit weights, we extend the standard Jaccard distance to a weighted form and employ a weighted variant of $k$-FreqItems \cite{huang2023kfreqitems} with a fixed base cluster count $k_0$ (see Appendix~\ref{app:impl:w_kfreq} for details).
This procedure yields per-round cluster labels $\ell_r(i) \in \{0,\cdots,k_0-1\}$ and sparse binary cluster centers (FreqItems) $\{\bm{c}^{(r)}_j\}_{j=1}^{k_0}$.

\paragraph{Stage II: Record-Space Refinement}
The $R$ base partitions provide complementary but noisy views. 
We summarize them in a record matrix
$\bm{L} \in \{0,\cdots,k_0-1\}^{n\times R}$ with $\bm{L}[i,r]=\ell_r(i)$.
We then one-hot encode $\bm{L}$ to obtain a binary embedding, where each round $r$ contributes a length-$k_0$ indicator block. A final $k$-FreqItems run on this embedding with $K$ clusters produces the final partition.

\subsection{Module 4: Intrinsic Interpretability}
\label{sec:method:interpretability}

Most clustering methods are not inherently interpretable, as assignments are driven by opaque objectives or latent embeddings with no direct linkage to original features.
In contrast, \textbf{WISE} is end-to-end transparent: BEP preserves feature semantics, LOFO-derived weights explicitly shape similarity, and the two-stage clustering exposes all intermediate labels.
This module quantifies the contribution of each weighted view and each instance to the final clusters, while guaranteeing consistency between instance- and cluster-level explanations.

Let $R$ be the number of weighted views, $k_0$ and $K$ be the cluster count in Stage~I and Stage~II respectively.
Stage~I yields per-round labels $\ell_r(i)$ for each instance $i \in \{1,\cdots,n\}$. Each round $r\in\{1,\cdots,R\}$ is associated with a feature-weight vector $\bm{w}^{(r)} \in \mathbb{R}_{\ge 0}^{d}$ with $\|\bm{w}^{(r)}\|_1=1$. Stage~II yields final clusters $\{\mathcal{C}_1, \cdots, \mathcal{C}_K\}$.
\begin{definition}[\textbf{Cluster-Bit Frequency}]
\label{def:clus-bit-freq}
  For a final cluster $\mathcal{C}_j$, round $r$, and label-bit $t$, define cluster-bit frequency: 
  \begin{displaymath}
    F_{j}(r,t) = \textstyle \frac{1}{|\mathcal{C}_j|}\sum_{i\in \mathcal{C}_j} \mathbb{I}\{\ell_r(i)=t\}.
  \end{displaymath}
\end{definition}

\begin{definition}[\textbf{Discriminative FreqItems (DFI)}]
\label{def:dfi}
  DFI measures how strongly a cluster $\mathcal{C}_j$ is over-represented at a given label-bit relative to competing clusters:
  \begin{displaymath}
    \mathrm{DFI}_{j}(r,t) = \textstyle \max\bigl\{0,\ F_{j}(r,t)-\max_{k\neq j}F_{k}(r,t)\bigr\}.
  \end{displaymath}
  The round credit for $\mathcal{C}_j$ is $c_{j,r} = \textstyle \sum_{t=0}^{k_0-1}\mathrm{DFI}_{j}(r,t)$.
\end{definition}

\begin{definition}[\textbf{Cluster-Level Feature Weights}]
\label{def:clus-feature-weights}
  Let $L_1(\bm{v})=\bm{v}/\|\bm{v}\|_1$ when $\|\bm{v}\|_1>0$.
  Cluster-level feature weight is obtained by aggregating round weights weighted by their discriminative credits:
  \begin{displaymath}
    \textstyle \bm{W}_{\mathrm{cluster}}(\mathcal{C}_j) = L_1(\sum_{r=1}^{R} c_{j,r}\,\bm{w}^{(r)}) \in \mathbb{R}_{\ge 0}^{d}.
  \end{displaymath}
\end{definition}

\begin{definition}[\textbf{Instance-Level DFI and Feature Weights}]
\label{def:instweights}
  Let $\varepsilon>0$ be a small constant.
  For an instance $i \in \mathcal{C}_j$, define its instance round contribution as $c_i[r] = \textstyle \frac{\mathrm{DFI}_{j}(r, \ell_r(i))}{\max\{\varepsilon, F_{j}(r,\ell_r(i))\}}$,
  and its instance-level feature weights as:
  \begin{displaymath}
    \bm{w}^{\mathrm{raw}}_i = \textstyle \sum_{r=1}^{R} c_i[r]\ \bm{w}^{(r)},
    \quad
    \bm{W}_{\mathrm{instance}}(i) = L_1(\bm{w}^{\mathrm{raw}}_i).
  \end{displaymath}
\end{definition}

\begin{lemma}[\textbf{Instance-to-Cluster Consistency}]
\label{lem:consistency}
  For any $\mathcal{C}_j$, averaging raw instance feature vectors over cluster members recovers the (pre-normalized) cluster vector:
  \begin{displaymath}
    \textstyle \frac{1}{|\mathcal{C}_j|}\sum_{i \in \mathcal{C}_j}\bm{w}^{\mathrm{raw}}_i = \sum_{r=1}^{R} c_{j,r}\,\bm{w}^{(r)},
  \end{displaymath}
  and L1-normalization recovers $\bm{W}_{\mathrm{cluster}}(\mathcal{C}_j)$:
  \begin{displaymath}
    \textstyle L_1(\frac{1}{|\mathcal{C}_j|} \sum_{i \in \mathcal{C}_j} \bm{w}^{\mathrm{raw}}_i) = \bm{W}_{\mathrm{cluster}}(\mathcal{C}_j).
  \end{displaymath}
\end{lemma}

\begin{proof}
  Fix a cluster $\mathcal{C}_j$ and a round $r$. Group instances in $\mathcal{C}_j$ by their round-$r$ label $t$:
  \begin{displaymath}
  \begin{aligned}
    \textstyle \frac{1}{|\mathcal{C}_j|}\sum_{i \in \mathcal{C}_j} c_i[r] = 
    & \textstyle \sum_{t=0}^{k_0-1} \frac{\mathrm{DFI}_{j}(r,t)}{\max\{\varepsilon, F_{j}(r,t)\}} \\
    & \textstyle \cdot \frac{1}{|\mathcal{C}_j|} \sum_{i \in \mathcal{C}_j}\mathbb{I}\{\ell_r(i)=t\}.
  \end{aligned}
  \end{displaymath}
  By Definition~\ref{def:clus-bit-freq}, the last factor is $F_{j}(r,t)$, hence
  \begin{displaymath}
    \textstyle \frac{1}{|\mathcal{C}_j|}\sum_{i \in \mathcal{C}_j} c_i[r] = 
    \sum_{t=0}^{k_0-1}\frac{\mathrm{DFI}_{j}(r,t)}{\max\{\varepsilon, F_{j}(r,t)\}}
    \cdot F_{j}(r,t).
  \end{displaymath}
  If $F_{j}(r,t)=0$, then $\mathrm{DFI}_{j}(r,t)=0$, and the term vanishes.
  For $F_{j}(r,t)>0$ and $\varepsilon \ll 1$, the denominator equals $F_{j}(r,t)$, giving
  $\textstyle \frac{1}{|\mathcal{C}_j|}\sum_{i \in \mathcal{C}_j} c_i[r] = \sum_{t=0}^{k_0-1}\mathrm{DFI}_{j}(r,t) = c_{j,r}$.
  Multiplying by $\bm{w}^{(r)}$ and summing over $r$ yields:
  \begin{displaymath}
  \begin{aligned}
    \textstyle \frac{1}{|\mathcal{C}_j|}\sum_{i \in \mathcal{C}_j}\bm{w}^{\mathrm{raw}}_i 
    &= \textstyle \sum_{r=1}^{R}\Bigl(\frac{1}{|\mathcal{C}_j|}\sum_{i \in \mathcal{C}_j} c_i[r]\Bigr)\bm{w}^{(r)} \\
    &= \textstyle \sum_{r=1}^{R} c_{j,r}\,\bm{w}^{(r)}.
  \end{aligned}
  \end{displaymath}
  L1-normalization on both sides completes the proof.
\end{proof}

\section{Experiments}
\label{sec:expt}

\subsection{Experimental Setup}
\label{sec:expt:setup}

\paragraph{Datasets}
We evaluate WISE on six real-world mixed-type tabular datasets: \textbf{Adult} \cite{uci_adult_1996}, \textbf{Vermont} \cite{urban_dp_vermont_2020}, \textbf{Arizona} \cite{urban_dp_arizona_2020}, \textbf{Obesity} \cite{palechor2019obesity}, \textbf{Credit} \cite{uci_credit_approval_1987}, and \textbf{GeoNames} \cite{geonames}. 
These datasets span diverse scales, feature dimensionality, and numerical-categorical composition, forming a comprehensive testbed for mixed-type clustering. 
Detailed statistics and preprocessing are provided in Appendix~\ref{app:impl:datasets}.

\paragraph{Evaluation Metrics}
We evaluate clustering using six complementary metrics capturing quality and efficiency.
When ground-truth is available, we report four external metrics: Adjusted Rand Index (\textbf{ARI}) \cite{hubert1985ARI}, Normalized Mutual Information (\textbf{NMI}) \cite{vinh2009NMI}, \textbf{Purity} \cite{manning2008ir}, and Clustering Accuracy (\textbf{ACC}), to quantify agreement with reference labels. 
We additionally report the Silhouette Coefficient (\textbf{SWC}) \cite{rousseeuw1987silhouette} for intrinsic structure and end-to-end wall-clock time for efficiency. 
Metric definitions are given in Appendix~\ref{app:impl:metrics}.

\paragraph{Baselines}
We compare \textbf{WISE} with five representative baselines covering classical, neural, and representation-learning approaches for mixed-type tabular clustering.
These include the classical $k$-Prototypes (\textbf{k-Proto}) \cite{huang1997kprototypes}; neural clustering methods \textbf{IDC} \cite{svirsky2024IDC}, \textbf{TableDC} \cite{rauf2025TableDC}, and \textbf{TELL} \cite{peng2022tell}; and a transformer-based representation baseline, \textbf{SAINT} \cite{somepalli2021saint} followed by $k$-Means, to assess whether representation learning alone suffices.
Hyperparameter settings and experiment environment are provided in Appendices \ref{app:impl:params} and \ref{app:impl:environment}, respectively.
\vspace{-0.5em}

\begin{table}[t]
\centering
\small
\renewcommand{\arraystretch}{1.0}
\caption{Clustering quality of all methods. \textbf{Bold} and \underline{underlined} denote the best and second-best results, respectively.}
\label{tab:quality}
\resizebox{0.99\columnwidth}{!}{
\begin{tabular}{clccccc}
\toprule
\rowcolor[HTML]{D9EAD3}
\textbf{Dataset} & \textbf{Method} & \textbf{ARI$\uparrow$} & \textbf{NMI$\uparrow$} & \textbf{Purity$\uparrow$} & \textbf{ACC$\uparrow$} & \textbf{SWC$\uparrow$} \\
\midrule
\multirow{6}{*}{\rotatebox[origin=c]{90}{\textbf{Adult}}}
 & \textbf{k-Proto} & 0.259 & 0.266 & 0.616 & \underline{0.475} & \underline{0.146} \\
 & \textbf{IDC}       & 0.010 & 0.052 & 0.446 & 0.380 & -0.042 \\
 & \textbf{TableDC}   & \underline{0.359} & \underline{0.412} & \textbf{0.682} & 0.454 & 0.116 \\
 & \textbf{TELL}      & 0.049 & 0.109 & 0.479 & 0.248 & 0.094 \\
 & \textbf{SAINT}     & 0.023 & 0.014 & 0.422 & 0.313 & -0.002 \\
 & \cellcolor[HTML]{FFF2CC}\textbf{WISE}      & \cellcolor[HTML]{FFF2CC}\textbf{0.663} & \cellcolor[HTML]{FFF2CC}\textbf{0.515} & \cellcolor[HTML]{FFF2CC}\underline{0.671} & \cellcolor[HTML]{FFF2CC}\textbf{0.633} & \cellcolor[HTML]{FFF2CC}\textbf{0.165} \\
\midrule

\multirow{6}{*}{\rotatebox[origin=c]{90}{\textbf{Vermont}}}
 & \textbf{k-Proto}  & 0.225 & 0.177 & 0.520 & 0.488 & \underline{0.151} \\
 & \textbf{IDC}       & 0.050 & 0.099 & 0.505 & 0.328 & 0.013 \\
 & \textbf{TableDC}   & 0.249 & \textbf{0.247} & 0.534 & \textbf{0.534} & \textbf{0.164} \\
 & \textbf{TELL}      & 0.213 & \underline{0.232} & \textbf{0.556} & 0.470 & 0.116 \\
 & \textbf{SAINT}     & \underline{0.252} & 0.210 & \underline{0.536} & 0.521 & 0.089 \\
 & \cellcolor[HTML]{FFF2CC}\textbf{WISE}      & \cellcolor[HTML]{FFF2CC}\textbf{0.283} & \cellcolor[HTML]{FFF2CC}0.223 & \cellcolor[HTML]{FFF2CC}0.535 & \cellcolor[HTML]{FFF2CC}\underline{0.534} & \cellcolor[HTML]{FFF2CC}0.131 \\
\midrule

\multirow{6}{*}{\rotatebox[origin=c]{90}{\textbf{Arizona}}}
 & \textbf{k-Proto}  & \underline{0.168} & \underline{0.149} & \textbf{0.607} & \underline{0.430} & 0.130 \\
 & \textbf{IDC}       & 0.004 & 0.016 & \underline{0.590} & 0.294 & 0.024 \\
 & \textbf{TableDC}   & 0.024 & 0.006 & \underline{0.590} & 0.390 & \underline{0.140} \\
 & \textbf{TELL}      & 0.013 & 0.009 & \underline{0.590} & 0.312 & 0.103 \\
 & \textbf{SAINT}     & 0.058 & 0.060 & \underline{0.590} & 0.398 & 0.121 \\
 & \cellcolor[HTML]{FFF2CC}\textbf{WISE}      & \cellcolor[HTML]{FFF2CC}\textbf{0.309} & \cellcolor[HTML]{FFF2CC}\textbf{0.321} & \cellcolor[HTML]{FFF2CC}\textbf{0.607} & \cellcolor[HTML]{FFF2CC}\textbf{0.607} & \cellcolor[HTML]{FFF2CC}\textbf{0.185} \\
\midrule

\multirow{6}{*}{\rotatebox[origin=c]{90}{\textbf{Obesity}}}
 & \textbf{k-Proto}  & \underline{0.198} & 0.228 & 0.389 & 0.370 & 0.124 \\
 & \textbf{IDC}       & 0.000 & 0.000 & 0.166 & 0.166 & 0.000 \\
 & \textbf{TableDC}   & 0.179 & \underline{0.240} & \underline{0.406} & \underline{0.383} & \textbf{0.228} \\
 & \textbf{TELL}      & 0.095 & 0.207 & 0.309 & 0.282 & 0.181 \\
 & \textbf{SAINT}     & 0.136 & 0.222 & 0.388 & 0.328 & -0.031 \\
 & \cellcolor[HTML]{FFF2CC}\textbf{WISE}      & \cellcolor[HTML]{FFF2CC}\textbf{0.222} & \cellcolor[HTML]{FFF2CC}\textbf{0.263} & \cellcolor[HTML]{FFF2CC}\textbf{0.441} & \cellcolor[HTML]{FFF2CC}\textbf{0.392} & \cellcolor[HTML]{FFF2CC}\underline{0.223} \\
\midrule

\multirow{6}{*}{\rotatebox[origin=c]{90}{\textbf{Credit}}}
 & \textbf{k-Proto}  & \underline{0.325} & \underline{0.245} & \underline{0.786} & \underline{0.786} & 0.184 \\
 & \textbf{IDC}       & 0.200 & 0.233 & 0.725 & 0.725 & 0.121 \\
 & \textbf{TableDC}   & 0.005 & 0.026 & 0.555 & 0.548 & \textbf{0.319} \\
 & \textbf{TELL}      & 0.233 & 0.188 & 0.742 & 0.742 & 0.215 \\
 & \textbf{SAINT}     & 0.051 & 0.076 & 0.622 & 0.622 & 0.113 \\
 & \cellcolor[HTML]{FFF2CC}\textbf{WISE}      & \cellcolor[HTML]{FFF2CC}\textbf{0.328} & \cellcolor[HTML]{FFF2CC}\textbf{0.285} & \cellcolor[HTML]{FFF2CC}\textbf{0.787} & \cellcolor[HTML]{FFF2CC}\textbf{0.787} & \cellcolor[HTML]{FFF2CC}\underline{0.234} \\
\midrule

\multirow{6}{*}{\rotatebox[origin=c]{90}{\textbf{GeoNames}}}
 & \textbf{k-Proto}  & 0.082 & 0.133 & 0.447 & 0.287 & 0.132 \\
& \textbf{IDC}       & 0.004 & 0.064 & 0.297 & 0.255 & 0.096 \\
 & \textbf{TableDC}   & \underline{0.112} & \underline{0.158} & \underline{0.462} & \underline{0.323} & \underline{0.302} \\
 & \textbf{TELL}      & 0.051 & 0.088 & 0.398 & 0.237 & 0.038 \\
 & \textbf{SAINT}     & 0.024 & 0.054 & 0.338 & 0.212 & 0.021 \\
 & \cellcolor[HTML]{FFF2CC}\textbf{WISE}      & \cellcolor[HTML]{FFF2CC}\textbf{0.146} & \cellcolor[HTML]{FFF2CC}\textbf{0.182} & \cellcolor[HTML]{FFF2CC}\textbf{0.491} & \cellcolor[HTML]{FFF2CC}\textbf{0.337} & \cellcolor[HTML]{FFF2CC}\textbf{0.368} \\
\bottomrule
\end{tabular}
}
\end{table}
\setlength{\textfloatsep}{1.5em}

\begin{table}[t]
\centering
\small
\renewcommand{\arraystretch}{1.0}
\caption{Average rank ($\downarrow$) across datasets. For each dataset and metric, methods are ranked from 1 (best) to 6 (worst); ties are resolved by assigning the average of the tied ranks.} 
\label{tab:avg_rank}
\begin{tabular}{lccccc}
  \toprule
  \rowcolor[HTML]{D9EAD3}
  \textbf{Method} & \textbf{ARI} & \textbf{NMI} & \textbf{Purity} & \textbf{ACC} & \textbf{SWC} \\
  \midrule
  \textbf{k-Proto} & \underline{2.67} & \underline{3.00} & \underline{2.92} & \underline{2.67} & \underline{3.00} \\
  \textbf{IDC}     & 5.67 & 4.83 & 5.25 & 5.00 & 5.33 \\
  \textbf{TableDC} & 3.33 & 3.17 & 3.25 & 3.08 & \textbf{1.67} \\
  \textbf{TELL}    & 4.33 & 4.00 & 3.58 & 4.83 & 4.00 \\
  \textbf{SAINT}   & 4.00 & 4.67 & 4.42 & 4.33 & 5.33 \\
  \rowcolor[HTML]{FFF2CC} 
  \textbf{WISE}    & \textbf{1.00} & \textbf{1.33} & \textbf{1.58} & \textbf{1.08} & \textbf{1.67} \\
  \bottomrule
\end{tabular}
\end{table}


\subsection{Clustering Quality}
\label{sec:expt:quality}

\paragraph{Overall Results} 
Table~\ref{tab:quality} summarizes results across five metrics and six datasets.
\textbf{WISE} achieves the most consistent gains on the external metrics (ARI, NMI, Purity, and ACC), indicating robust alignment with ground-truth semantics across diverse datasets.
Its advantages are most pronounced in challenging regimes with high dimensionality, complex categorical structure, or correlated features, where single global trade-offs or uniform weighting fail to capture fine-grained semantics.
By contrast, WISE's LOFO-based weighting emphasizes jointly informative subsets, yielding gains on ARI, NMI, and ACC rather than Purity alone.

We also observe a recurring gap between intrinsic compactness and semantic alignment.
Deep clustering baselines (e.g., \textbf{TableDC} and \textbf{TELL}) often achieve high compactness (SWC) yet align poorly with ground truth, underscoring the limits of latent geometry in mixed-type data.
\textbf{WISE} instead grounds similarity in feature-level relevance, producing clusters that are both coherent and semantically meaningful.


\begin{figure*}[t]
  \centering
  \includegraphics[width=0.99\textwidth]{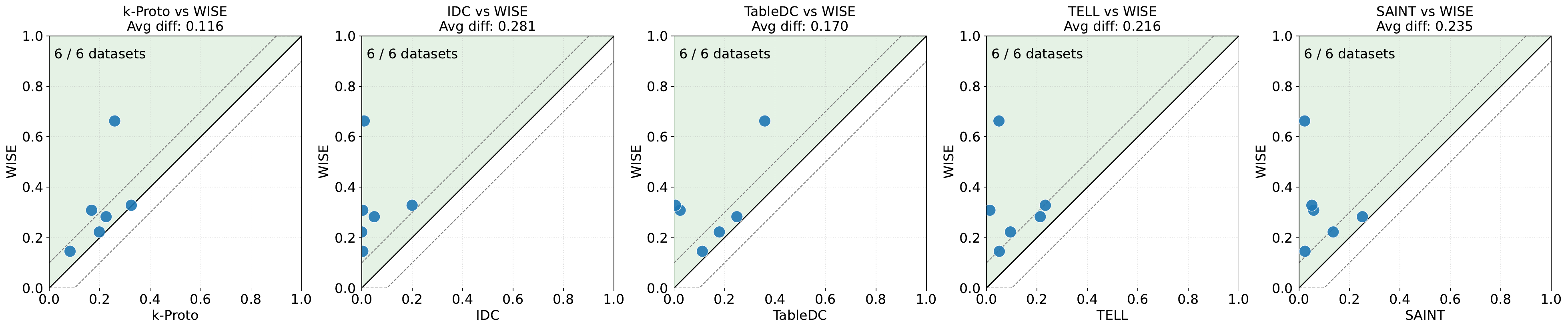}
  \vspace{-0.5em}
  \caption{Pairwise ARI comparisons between WISE and each baseline.}
  \label{fig:pairwise_ari}
  \vspace{-0.5em}
\end{figure*}

\begin{figure*}[t]
  \centering
  \includegraphics[width=0.99\textwidth]{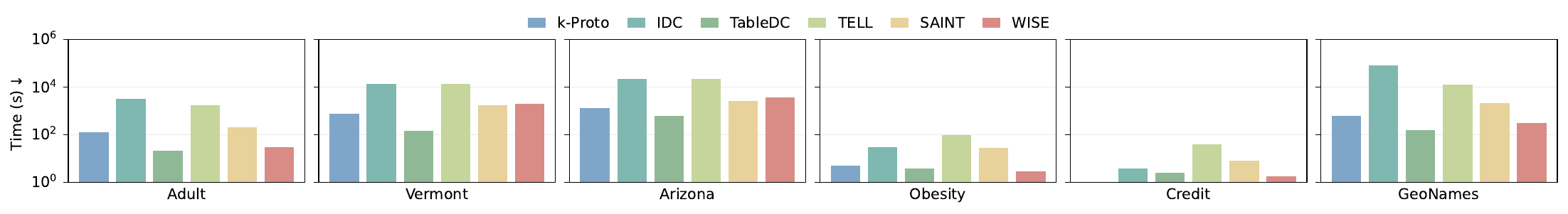}
  \vspace{-0.5em}
  \caption{Clustering efficiency of all methods on each dataset, measured on the same hardware and software stack. Reported times (seconds; log scale) include the complete clustering pipeline, covering preprocessing, representation learning (if any), and clustering.}
  \label{fig:runtime}
  \vspace{-1.0em}
\end{figure*}

\paragraph{Ranking Analysis}
Table~\ref{tab:avg_rank} reports average ranks across datasets.
\textbf{WISE} achieves the best rank on all external metrics and ties for the best on SWC, showing consistent gains rather than reliance on any single dataset or criterion.
This stability is crucial for mixed-type clustering, where datasets stress different failure modes, such as correlated attributes (Adult), high categorical complexity (Arizona), or many weak or irrelevant features (GeoNames, Credit).

The rankings also reveal complementary baseline strengths:
\textbf{k-Proto} is a reliable runner-up on agreement metrics but lacks adaptability to uneven feature relevance, while \textbf{TableDC} often excels on SWC, yet this compactness advantage does not reliably translate to better semantic agreement.
Overall, \textbf{WISE}'s top ranks support the core claim that granularity-aligned representation with weight-informed clustering yields the most reliable performance.

\paragraph{Pairwise Comparisons}
For a finer-grained view, Figure~\ref{fig:pairwise_ari} presents dataset-wise ARI scatter plots contrasting \textbf{WISE} with each baseline. 
Each point denotes a dataset; points above the diagonal indicate improvements by WISE, with the dashed line marking a +10\% margin.
\textbf{WISE} consistently outperforms every baseline on all datasets. 
The average ARI gains are substantial, ranging from 0.116 over \textbf{k-Proto} to 0.281 over \textbf{IDC}, with similarly consistent improvements over \textbf{TableDC}, \textbf{TELL}, and \textbf{SAINT}.
These uniform wins reinforce that \textbf{WISE}'s advantages are systematic rather than dataset-specific, directly reflecting the benefits of sensing diverse feature weightings and aggregating them through a principled, weight-informed clustering mechanism.

\subsection{Clustering Efficiency}
\label{sec:expt:efficiency}

Figure~\ref{fig:runtime} summarizes end-to-end wall-clock times.
\textbf{WISE} is consistently faster than deep clustering baselines, achieving 2$\sim$250$\times$ speedups over \textbf{IDC} and 6$\sim$57$\times$ over \textbf{TELL}, while remaining broadly competitive with \textbf{k-Proto}.
Compared to \textbf{SAINT}, WISE is several times faster on four datasets and slower on Vermont and Arizona, where additional cost arises from feature-wise weight sensing and multi-view clustering; these steps are embarrassingly parallel and can be substantially accelerated in practice.
\textbf{TableDC} is often the fastest overall, but WISE typically operates within the same practical regime (about 2$\sim$10$\times$ overhead) while delivering stronger clustering quality and intrinsic, feature-level explanations.
Overall, WISE's accuracy and interpretability gains come without prohibitive computational cost.
\vspace{-0.5em}

\subsection{Ablation Study}
\label{sec:expt:ablation}

\begin{figure}[t]
  \centering
  \includegraphics[width=0.99\columnwidth]{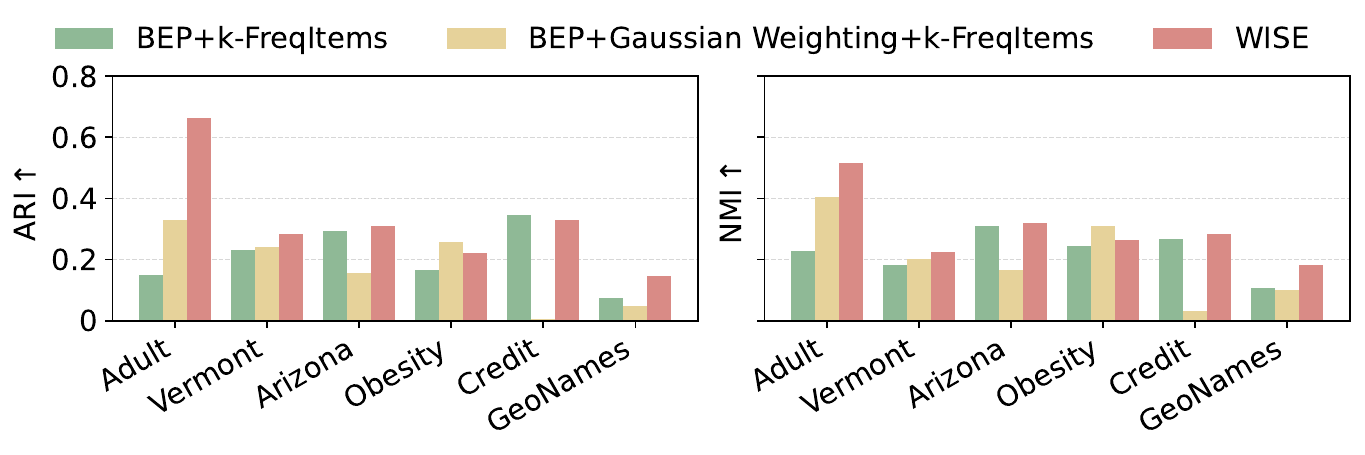}
  \vspace{-0.5em}
  \caption{Ablation study of WISE evaluating modules 2 and 3.} 
  \label{fig:ablation}
  \vspace{-1.0em}
\end{figure}

\begin{figure}[t]
  \centering
  \includegraphics[width=0.99\columnwidth]{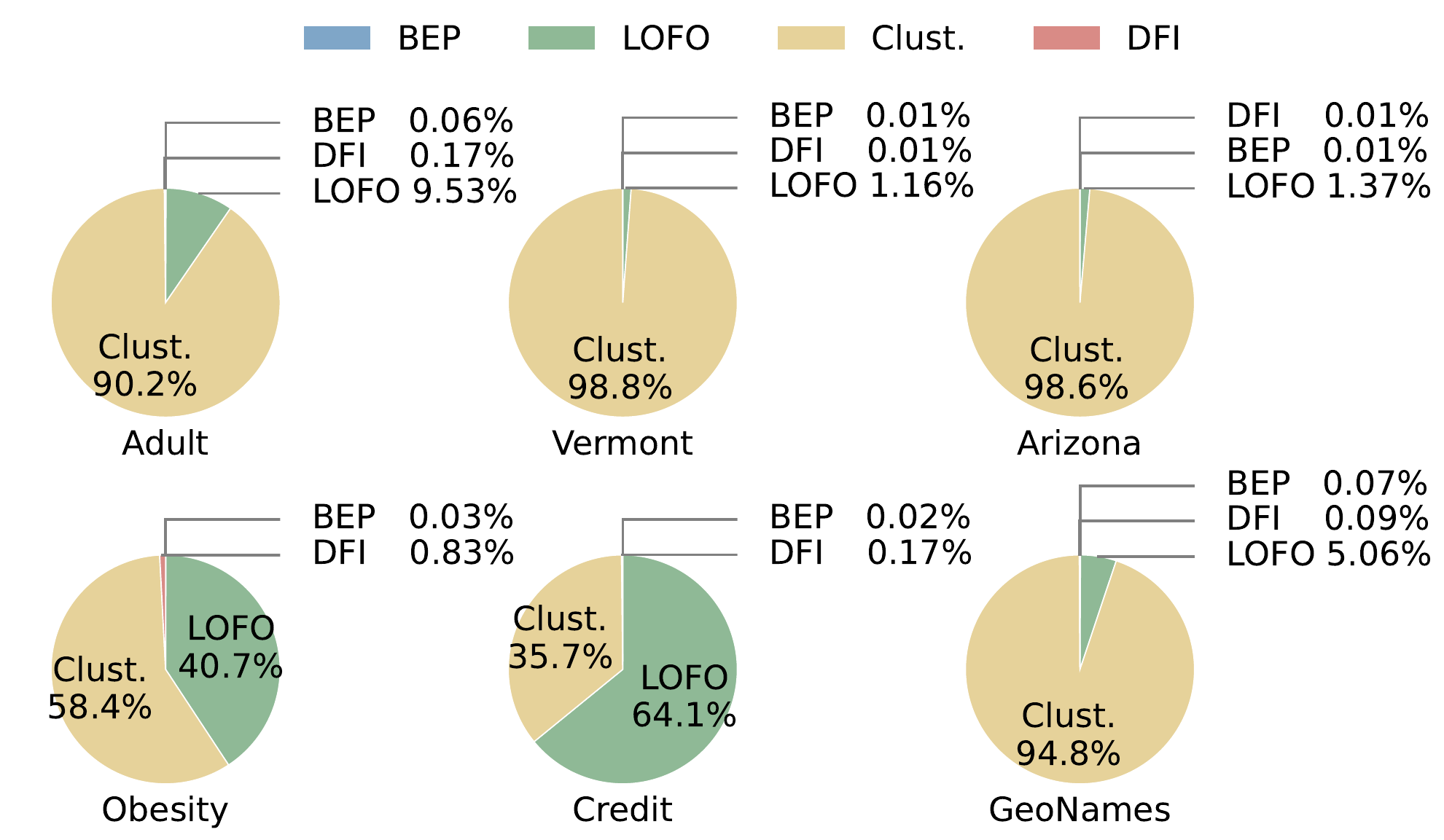}
  \vspace{-0.5em}
  \caption{Runtime breakdown of WISE across its four modules.}
  \label{fig:ablation_cost_breakdown}
  \vspace{-0.5em}
\end{figure}

We evaluate the contribution of \textbf{WISE}'s LOFO-based weight sensing and weight-informed clustering by comparing it against two baselines: \textbf{BEP+$k$-FreqItems} and \textbf{BEP+Gaussian Weighting+$k$-FreqItems}.
$k$-FreqItems is designed for sparse binary data; BEP serves as a fair adapter by mapping mixed-type tables into a unified sparse representation.
The Gaussian Weighting variant preserves WISE's two-stage pipeline but replaces LOFO with randomly sampled Gaussian weights, isolating the effect of data-driven sensing. 
All methods use the same number of clusters $K$.

As shown in Figure~\ref{fig:ablation}, \textbf{WISE} yields the most stable gains, improving NMI on all datasets and ARI on most, while remaining competitive elsewhere.
In contrast, the \textbf{Gaussian Weighting} variant shows inconsistent behavior: although random weights can help in some cases, they often amplify irrelevant dimensions and even underperform \textbf{BEP+$k$-FreqItems}. 
These results confirm that \emph{weighting alone is insufficient}; effective clustering requires weights that are \emph{aligned with true feature relevance}.

\paragraph{Runtime Breakdown}
Figure~\ref{fig:ablation_cost_breakdown} decomposes the end-to-end runtime of \textbf{WISE} across its four modules.
The cost is dominated by Module 2 (LOFO-based feature weight sensing) and Module 3 (weight-informed clustering), which constitute the core modeling components of the framework. 
In contrast, BEP encoding and DFI explanation incur negligible overhead.
This confirms that \textbf{WISE} concentrates computation on its essential weighting and aggregation mechanisms, consistent with the design goals in Section \ref{sec:method}.

\begin{figure}[t]
  \centering
  \begin{subfigure}[t]{\columnwidth}
    \centering
    \includegraphics[width=0.99\linewidth]{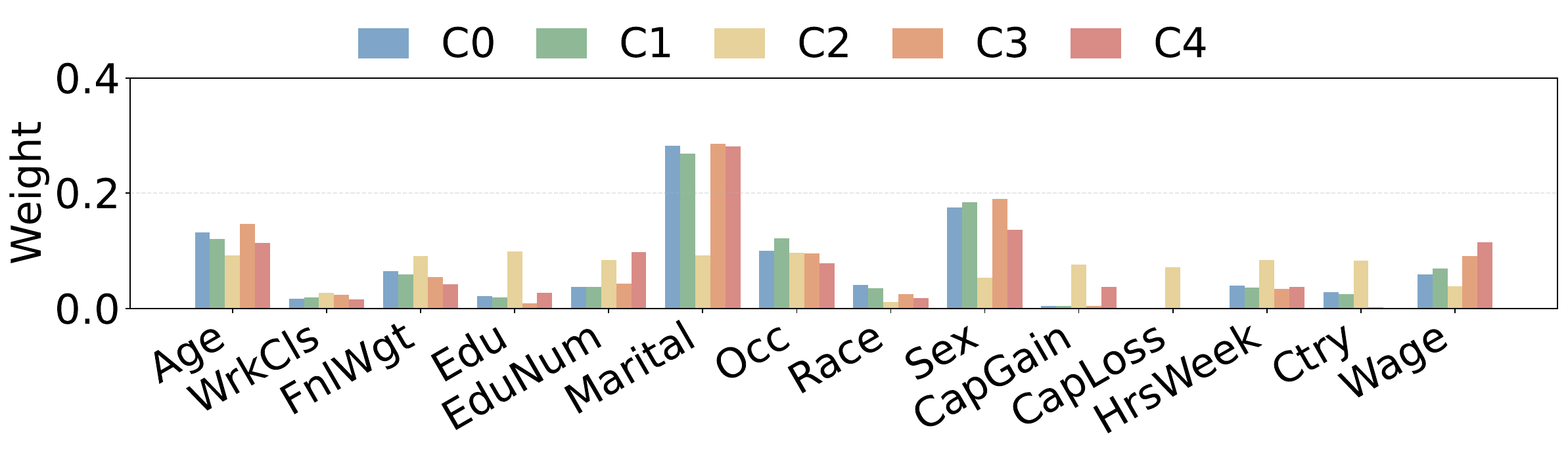}
    \vspace{-0.5em}
    \caption{Cluster-level feature weights.}
    \label{fig:adult-interp-a}
  \end{subfigure}
  \noindent
  \begin{minipage}{\columnwidth}
    \raggedright
    \begin{subfigure}[b]{0.49\columnwidth}
      \centering
      \includegraphics[width=\linewidth]{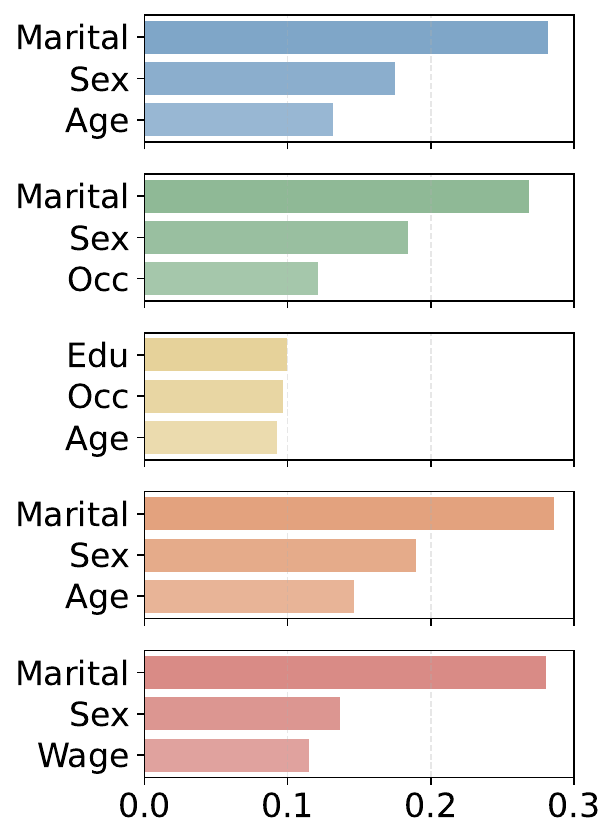}
      \caption{Top-3 features per cluster.}
      \label{fig:adult-interp-b}
    \end{subfigure}
    \begin{subfigure}[b]{0.45\columnwidth}
      \centering
      \includegraphics[width=\linewidth,trim=0 0 0 35,clip]{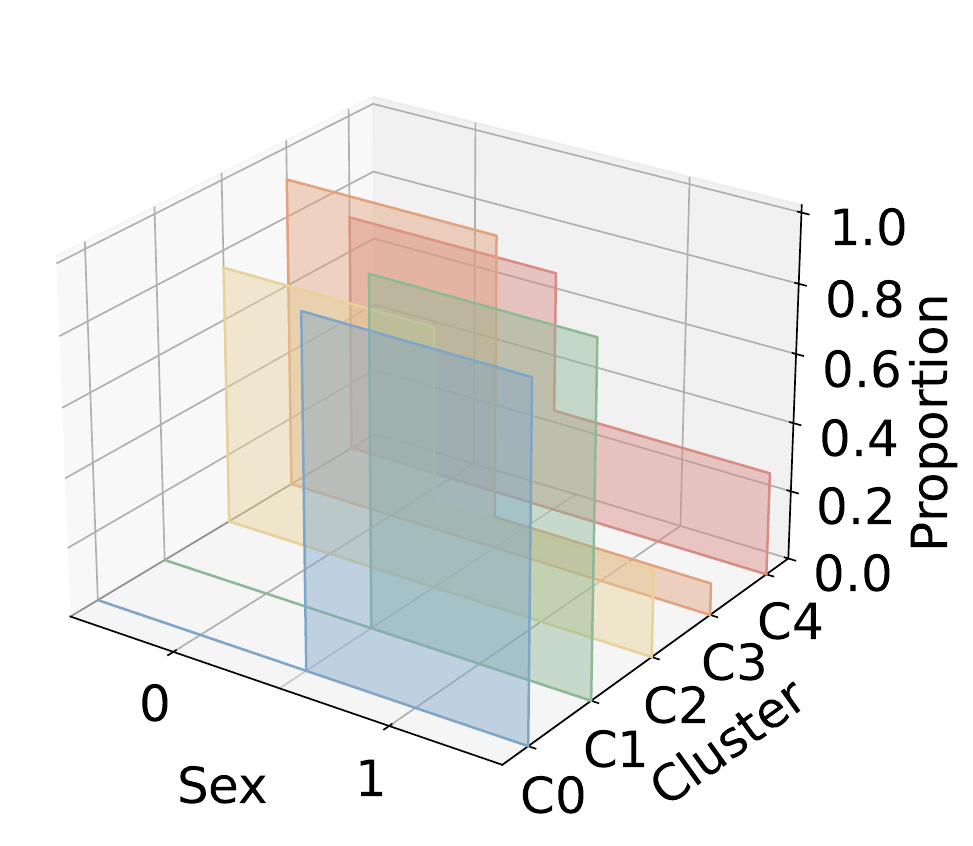}
      \includegraphics[width=\linewidth,trim=0 0 0 48,clip]{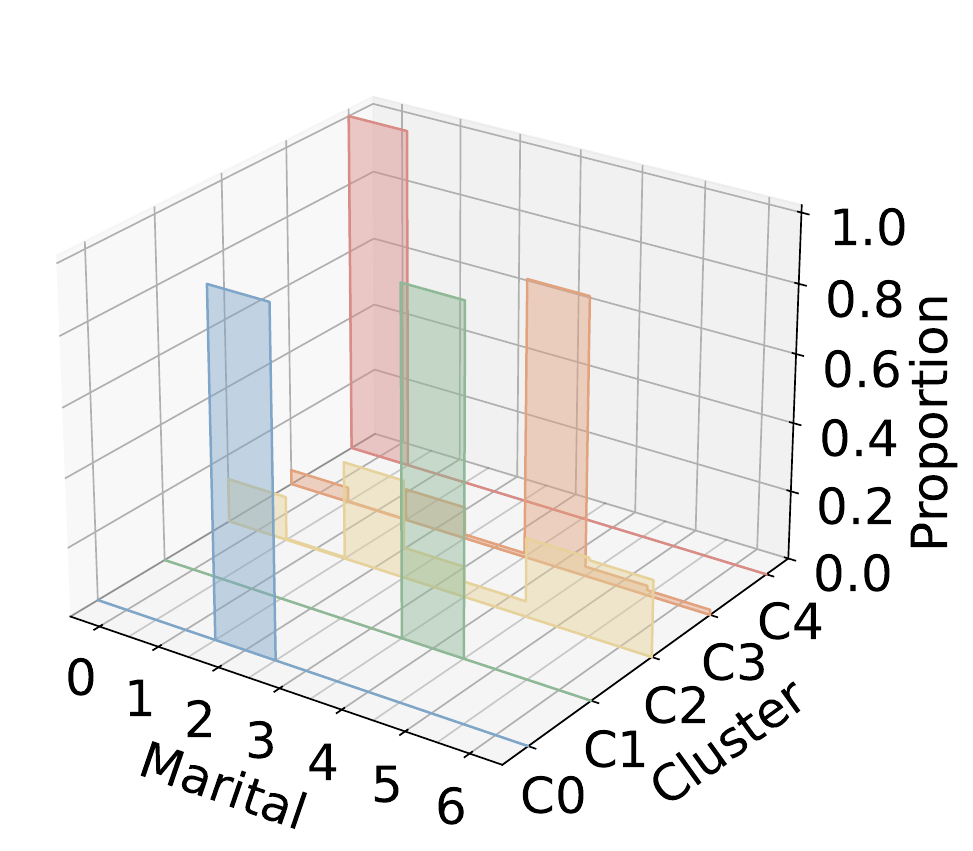}
      \caption{Cluster-wise histogram.}
      \label{fig:adult-interp-c}
    \end{subfigure}
  \end{minipage}
  \vspace{-0.5em}
  \caption{Case Study on Adult.}
  \label{fig:case_study_adult}
  \vspace{-0.5em}
\end{figure}

\begin{figure}[t]
  \centering
  \begin{subfigure}[t]{\columnwidth}
    \centering
    \includegraphics[width=0.99\linewidth]{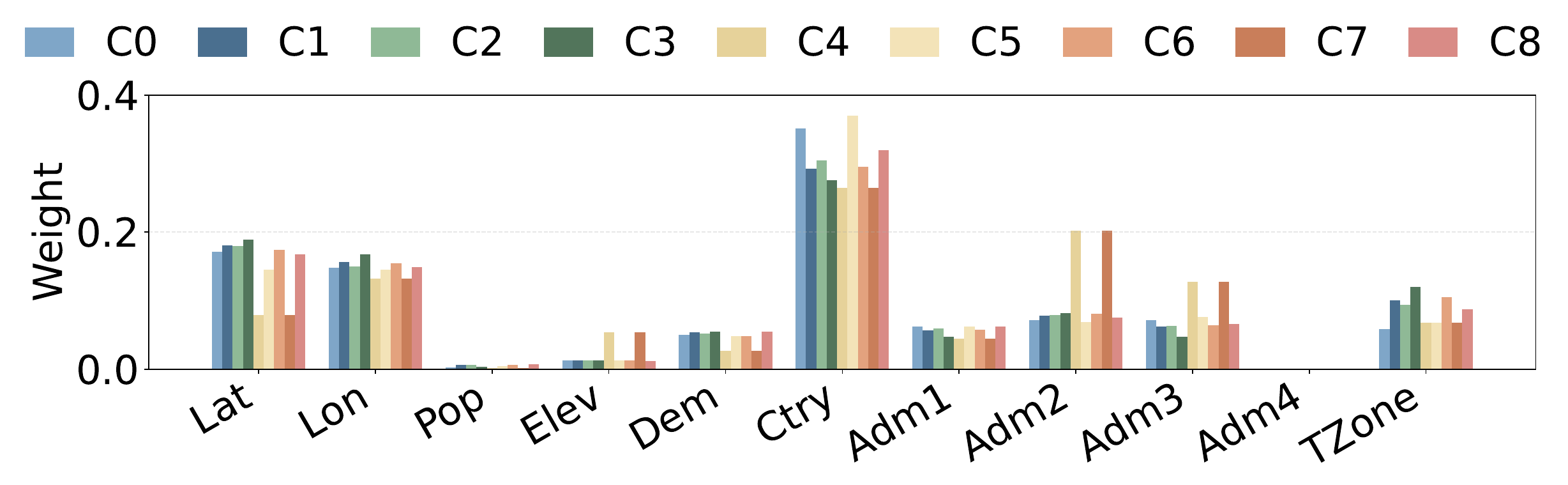}
    \vspace{-0.5em}
    \caption{Cluster-level feature weights.}
    \label{fig:geonames-interp-a}
  \end{subfigure}
  \begin{subfigure}[t]{\columnwidth}
    \centering
    \includegraphics[width=\linewidth]{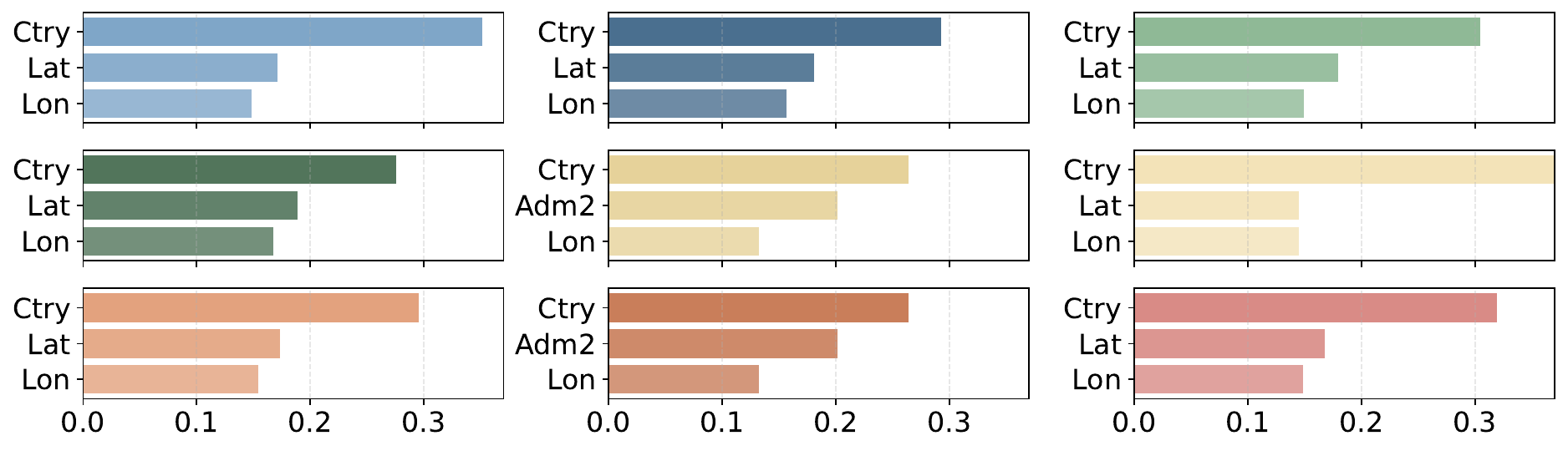}
    \caption{Top-3 feature per cluster.}
    \label{fig:geonames-interp-b}
  \end{subfigure}
  \begin{subfigure}[t]{\columnwidth}
    \centering
    \begin{minipage}[t]{0.495\columnwidth}
      \centering
      \includegraphics[width=\linewidth,trim=0 0 0 40,clip]{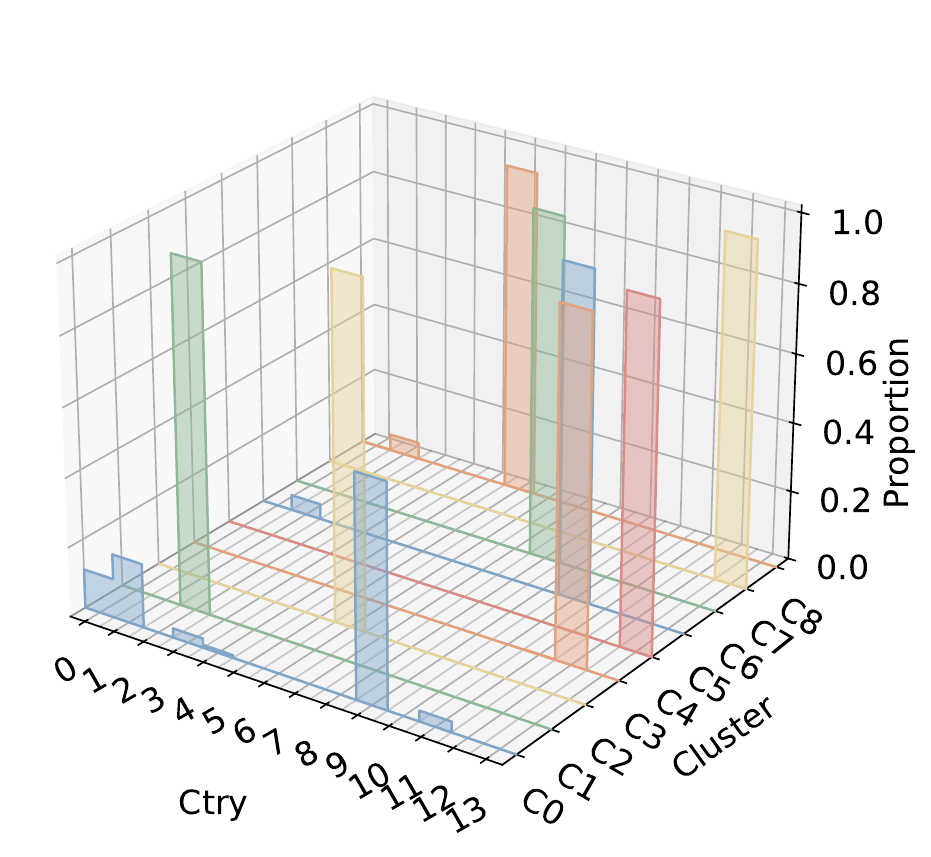}
    \end{minipage}
    \hfill
    \begin{minipage}[t]{0.495\columnwidth}
      \centering
      \includegraphics[width=\linewidth,trim=0 0 0 40,clip]{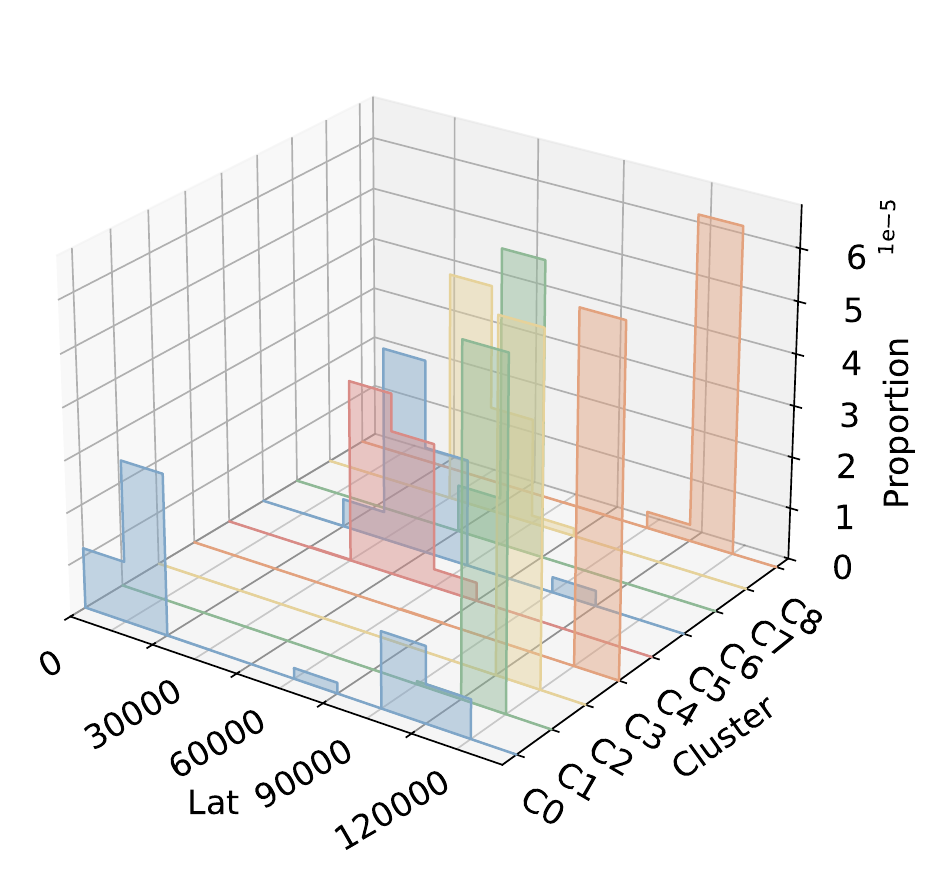}
    \end{minipage}
    \caption{Cluster-wise histogram.}
    \label{fig:geonames-interp-c}
  \end{subfigure}
  \vspace{-0.5em}
  \caption{Case Study on GeoNames.}
  \label{fig:case_study_geonames}
  \vspace{-0.5em}
\end{figure}

\subsection{Case Study} 
\label{sec:expt:case_study}

To showcase the intrinsic interpretability of \textbf{WISE}, we present case studies on Adult and GeoNames (Figures \ref{fig:case_study_adult} and \ref{fig:case_study_geonames}), using the exact cluster assignments from Section \ref{sec:expt:quality}. 
Each figure contains three components: 
(a) \textbf{cluster-level feature weights} derived from DFI, 
(b) the \textbf{top-3 discriminative features} per cluster, and 
(c) \textbf{cluster-wise feature histograms}, validating explanations against the raw data.

\paragraph{Case Study on Adult} 
The Adult dataset illustrates that feature relevance in \textbf{WISE} is explicitly \emph{cluster-dependent}.
As shown in Figures \ref{fig:adult-interp-a} and \ref{fig:adult-interp-b}, clusters rely on distinct feature subsets to separate from others.
For example, clusters \textbf{C0} and \textbf{C1} are primarily distinguished by \texttt{Marital-status} and \texttt{Sex}, together with \texttt{Age} (C0) or \texttt{Occupation} (C1). 
In contrast, cluster \textbf{C3} is more strongly characterized by \texttt{Edu} and \texttt{Occupation}, indicating a different semantic grouping.
These explanations are corroborated by the histograms in Figure~\ref{fig:adult-interp-c}. 
For \texttt{Marital-status} and \texttt{Sex}, clusters with higher DFI weights exhibit clearly separated distributions, while clusters with lower weights show substantial overlap.
This alignment between learned weights and empirical distributions confirms that DFI captures genuine discriminative structure rather than post-hoc artifacts.

\paragraph{Case Study on GeoNames}
A similar pattern emerges for GeoNames. 
Figures~\ref{fig:geonames-interp-a} and~\ref{fig:geonames-interp-b} show that \texttt{Latitude}, \texttt{Longitude}, and \texttt{Country\_code} dominate most clusters, reflecting the primacy of geographic location. 
However, certain clusters (e.g., \textbf{C4} and \textbf{C7}) assign higher importance to  \texttt{Admin2\_code}, indicating that finer-grained administrative features are necessary when country-level signals alone are insufficient.
Figure~\ref{fig:geonames-interp-c} validates these findings: while many clusters form distinct peaks on \texttt{Country\_code} and \texttt{Latitude}, \textbf{C4} and \textbf{C7} exhibit notable overlap on these features, explaining their reduced discriminative weights and the increased reliance on administrative attributes.


\section{Conclusions}
\label{sec:conclusions}

We introduced \textbf{WISE}, a fully unsupervised and intrinsically interpretable framework for clustering mixed-type tabular data. 
WISE integrates granularity-aligned representation, data-driven feature weighting, weight-informed clustering, and explanation into a unified pipeline that operates directly on feature-level semantics rather than opaque latent spaces.
Extensive experiments on six real-world datasets show that WISE consistently outperforms classical and neural baselines on external agreement metrics, remains computationally efficient, and produces faithful, human-interpretable explanations that reflect true discriminative structure.

\section*{Impact Statement}
This work improves unsupervised clustering for mixed-type tabular data by providing a transparent pipeline that aligns heterogeneous features, learns data-driven feature weights, and produces intrinsic, feature-level explanations of discovered cohorts. Such capabilities can benefit exploratory analysis in domains where tabular records are prevalent by enabling more faithful subgroup discovery, better hypothesis generation, and more auditable downstream use of clusters.

Potential negative impacts arise when clustering is applied to sensitive populations or attributes: faster and more interpretable cohorting may facilitate profiling, targeted manipulation, or discriminatory segmentation, and spurious or biased clusters may be mistaken as meaningful structure and propagated into downstream decisions. We recommend deploying the method with appropriate access controls and privacy safeguards, and conducting careful validation for stability, bias, and failure cases (especially when clusters inform high-stakes decisions), including scrutiny of sensitive features and human oversight of any consequential use.


\bibliography{reference}
\bibliographystyle{icml2026}

\newpage
\appendix
\onecolumn

\section{Implementation Details}
\label{app:impl}

\subsection{Binary Encoding with Padding}
\label{app:impl:bep}

BEP assigns a per-column bit budget $B$, but the within-column encoding depends on the semantic type of the feature. In particular, we explicitly distinguish ordinal from nominal categorical attributes.

\paragraph{Numerical and ordinal features.}
For numerical features, we use the positional-padding scheme: after scaling to the column range, a value is encoded by a contiguous block of $B$ ones whose offset reflects its relative magnitude. We treat ordinal categorical features in exactly the same way. Specifically, when a categorical column has a natural order available from domain semantics or dataset documentation, we first map its levels to ordered integers and then pass the resulting scalar values through the numerical BEP encoder. This preserves ordinal proximity in the BEP space: adjacent levels have overlapping blocks, while distant levels have less overlap or none. In other words, ordinal variables are not one-hot encoded in WISE; they are processed by the same BEP mechanism as numerical features.

\paragraph{Nominal categorical features.}
For nominal features, no intrinsic order is assumed, so categories are treated as equidistant. When a nominal feature $x_j$ has $c_j$ distinct categories and $c_j \leq B$, we use a collision-free one-hot block for column $j$. When $c_j > B$, a collision-free length-$B$ representation is impossible, and we support two practical treatments in our implementation.

\paragraph{Option 1: Hash-BEP}
We map categories to $B$ buckets using a hash function $h_j(\cdot)\in\{0,\dots,B-1\}$ and encode category $a$ by activating the single coordinate $h_j(a)$.
This keeps the BEP dimensionality and runtime predictable while allowing collisions between distinct categories that hash to the same bucket (the standard trade-off in feature hashing) \cite{weinberger2009feature}.

\paragraph{Option 2: Expand-BEP}
We allocate one bit per category and use a length-$c_j$ one-hot block for column $j$, i.e., we allow this column to exceed the nominal budget $B$.
This preserves category identity exactly, but increases the overall BEP dimensionality and can raise memory/time costs downstream.

These choices affect only the within-column encoding and do not change the rest of the WISE pipeline. Ordinal columns still pass through the same numerical BEP encoder, while nominal columns under standard one-hot, Hash-BEP, or Expand-BEP all produce sparse binary set vectors. As a result, the subsequent weighted-Jaccard computations and the weighted $k$-FreqItems backbone operate unchanged; only the coordinate allocation differs. Moreover, Hash-BEP preserves sparsity and has controlled distortion guarantees, suggesting that occasional collisions have limited effect in practice when $B$ is moderately sized \cite{weinberger2009feature}.

\subsection{From TreeSHAP to Feature Weight Vector}
\label{app:impl:treeshap}

We explain how WISE uses TreeSHAP \cite{lundberg2020treeshap} to convert Leave-One-Feature-Out (LOFO) prediction models into feature weight vectors.

For each target feature $x_j$ we construct a LOFO prediction problem with target $y \leftarrow x_j$ and inputs $X \leftarrow X_{-j}$. We train a Random Forest (RF) model and view each tree with index $u$ in the ensemble as a ``local evaluator'' of how other features contribute to predicting $x_j$.
We then use TreeSHAP to quantify, for this tree $u$, which input features in $X_{-j}$ are most responsible for its predictions.
Given an input row $\bm{x}_i$ with the $j$-th feature $f_j$ removed, TreeSHAP returns per-feature attributions $\bm{\phi}^{(j,u)}(\bm{x}_i) \in \mathbb{R}^{d-1}$ satisfying the standard additive decomposition
\begin{displaymath}
    f_{j,u}(\bm{x}_i)
    \;=\;
    \mathbb{E}\!\left[f_{j,u}(\bm{X})\right]
    \;+\;
    \sum_{k\neq j}\phi^{(j,u)}_{k}(\bm{x}_i),
\end{displaymath}
where $f_{j,u}$ denotes the prediction function of tree $u$ in the LOFO model for target $x_j$ \cite{lundberg2017shap, lundberg2018consistent, lundberg2020treeshap}.
In practice, we compute TreeSHAP using \texttt{shap.TreeExplainer} with \texttt{feature\_perturbation="interventional"} and a small background set sampled from training data, which defines the reference distribution used in the Shapley expectations \cite{shapdocs_treeexplainer}.

WISE needs a global nonnegative dependency signature per tree, rather than per-instance explanations. We therefore aggregate local TreeSHAP values over a set of explained rows $E$ by taking the mean absolute attribution:
\begin{displaymath}
    \bm{s}_{j,u}(k)
    \;=\;
    \frac{1}{|E|}
    \sum_{\bm{x}_i\in E}
    \left|\phi^{(j,u)}_{k}(\bm{x}_i)\right|
    \quad \in \mathbb{R}_{\ge 0}^{d-1},
\end{displaymath}
for each input feature with index $k\neq j$.
This follows the standard ``local-to-global'' usage pattern: combine many faithful local explanations to summarize model structure at the population level \citep{lundberg2020treeshap}.

\subsection{Quality-Diversity Greedy Search}
\label{app:impl:qd-greedy}
Given the Quality-Diversity objective \cite{gollapudi2009result-diversification, borodin2017maxsum}:
\begin{displaymath}
  J_j(S)= \frac{\lambda }{|S|} \sum_{u\in S} q_{j,u} + \frac{(1-\lambda) \cdot 2}{|S|(|S|-1)} \sum_{u<v \in S} (1-\cos(u,v)),
\end{displaymath}
we maximize $J_j$ via greedy search, starting from the highest quality tree
$\mathcal{U}_j^{(1)}=\{\arg\max_u q_{j,u}\}$ and iteratively adding the tree with the largest marginal gain until $|\mathcal{U}_j|=m$.
\begin{displaymath}
  u^\star = \arg\max_{r \notin \mathcal{U}_j^{(t)}} \Delta J_j \bigl(r;\mathcal{U}_j^{(t)}\bigr),~\mathcal{U}_j^{(t+1)} = \mathcal{U}_j^{(t)} \cup \{u^\star\}.
\end{displaymath}

For efficient evaluation, let $k=|S|$, $Q(S) = \sum_{u \in S} q_{j,u}$, $P(S) = \sum_{u<v \in S} 1-\cos(u,v)$, and $A(r;S)=\sum_{u\in S} 1-\cos(u,r)$. 
For any candidate $r\notin S$, the objective after inclusion admits the closed form
\begin{displaymath}
  J_j(S\cup\{r\}) = \lambda \cdot \frac{Q(S)+q_{j,r}}{k+1} + (1-\lambda) \cdot \frac{2}{k(k+1)}\Big(P(S)+A(r;S)\Big),
\end{displaymath}
enabling the marginal gain $\Delta J_j(r;S)=J_j(S\cup\{r\})-J_j(S)$ to be computed in $O(k)$ time, hence the greedy process is done in $O(k^2)$ time.

\subsection{Weighted $k$-FreqItems}
\label{app:impl:w_kfreq}
\paragraph{From mixed-type tabular data to sparse set data}
After the BEP procedure, each recode is represented as a high-dimensional sparse binary vector $\bm{x}\in\{0,1\}^d$, or equivalently, a set of active coordinates. This allows us to use $k$-FreqItems, a scalable Lloyd-style clustering method for sparse set data under Jaccard distance \cite{huang2023kfreqitems}.
For two binary set-vectors $\bm{x}, \bm{y}\in\{0,1\}^d$ with supports $X=\{t:x_t=1\}$ and $Y=\{t:y_t=1\}$, their Jaccard similarity and distance are
\begin{displaymath}
    J(\bm{x}, \bm{y}) \;=\ \frac{|X\cap Y|}{|X\cup Y|},
    \qquad
    d_J(\bm{x}, \bm{y}) \;=\; 1 - J(\bm{x},\bm{y}).
\end{displaymath}

\paragraph{FreqItem: A Sparse Center Representation}
Given a cluster $\mathcal{S}\subseteq \mathcal{X}$, $k$-FreqItems represents its center by a sparse set of frequent coordinates, called a \textbf{FreqItem} \cite{huang2023kfreqitems}.
Let $f_t = |\{ \bm{x}\in\mathcal{S}: x_t=1\}|$ be the frequency of coordinate $t$ within $\mathcal{S}$, and
$f_{\max}=\max_t f_t$.
With a sparsity control parameter $\alpha\in[0,1]$, the (unweighted) FreqItem center keeps coordinates whose
frequency is at least a fraction of the maximum:
\begin{displaymath}
    \mathrm{FreqItem}_\alpha(\mathcal{S})
    \;=\;
    \bigl\{t:\; f_t \ge \max(1,\lceil \alpha f_{\max}\rceil)\bigr\}.
    \label{eq:freqitem_binary}
\end{displaymath}
This interpolates between a dense ``Mode1''-style center ($\alpha=0$) and a very sparse ``Mode2''-style center
($\alpha\rightarrow 1$) while remaining efficient to compute on sparse data.

\paragraph{$k$-FreqItems Iterations}
Given $K$ centers $\{\bm{c}_1,\dots,\bm{c}_K\}$ represented as FreqItems, the algorithm alternates: 
(i) assignment by nearest-center Jaccard distance;
and (ii) update by recomputing each
center as the FreqItem of its assigned cluster.

\paragraph{SILK Seeding}
Initialization is performed by \textbf{SILK}, which first overseeds via a two-level MinHash procedure and then reduces the overseeded set to $K$ centers using a $k$-means$\parallel$-style reclustering \cite{bahmani2012scalable}.
Concretely, MinHash-based LSH produces buckets of similar data; a second MinHash pass groups buckets into bins, from which frequent data subsets are extracted, deduplicated, and converted into initial centers by FreqItem computation. The resulting center set is then downsampled to $K$ centers before the standard $k$-FreqItems iterations proceed.

\paragraph{Why We Need Weights}
WISE senses diverse feature weightings via LOFO and bias clustering toward important features in the first-stage clustering.
Since the original $k$-FreqItems is defined for unweighted set data, we extend it to a \textbf{weighted} variant by modifying the distance, the LSH used in SILK seeding, and the center update.

\paragraph{Weighted Set Representation}
We associate each coordinate $t\in\{1,\dots,d\}$ with a nonnegative weight $\omega_t\ge 0$ derived from WISE.
Each record becomes a sparse nonnegative vector $\bm{v}\in\mathbb{R}_+^d$ where
\begin{displaymath}
v_t = \omega_t \cdot x_t \quad (\text{so } v_t\in\{0,\omega_t\} \text{ under binary supports}).
\end{displaymath}

\paragraph{Weighted Jaccard}
For two weighted vectors $\bm{v},\bm{u}\in\mathbb{R}_+^d$, the weighted Jaccard similarity is
\begin{equation}
J_w(\bm{v},\bm{u})
\;=\;
\frac{\sum_{t=1}^d \min(v_t,u_t)}{\sum_{t=1}^d \max(v_t,u_t)},
\qquad
d_w(\bm{v},\bm{u}) = 1 - J_w(\bm{v},\bm{u}).
\label{eq:weighted_jaccard}
\end{equation}

\paragraph{Weighted MinHash}
Standard MinHash produces signatures whose collision probability equals Jaccard similarity and thus enables LSH for set similarity \cite{broder1998min-wise}. SILK relies on MinHash-based LSH to efficiently propose candidate centers from massive sparse data \cite{huang2023kfreqitems}.
To extend SILK to weighted Jaccard, we replace MinHash with Weighted MinHash via Consistent Weighted Sampling (CWS), where hash collisions match weighted Jaccard similarity \cite{manasse2010consistent, ioffe2010improved}. A common CWS instantiation samples, for each coordinate $t$ and hash function $h$, random variables $r_{t,h}\sim \mathrm{Gamma}(2,1)$, $c_{t,h}\sim \mathrm{Gamma}(2,1)$, and
$\beta_{t,h}\sim \mathrm{Uniform}(0,1)$, and computes a key $a_{t,h}$ from $v_t$; the coordinate achieving the
minimum key becomes the hash output. CWS satisfies
\begin{displaymath}
    \Pr\bigl[h(\bm{v}) = h(\bm{u})\bigr] \;=\; J_w(\bm{v},\bm{u}),
\end{displaymath}
making it an LSH primitive for weighted Jaccard \cite{ioffe2010improved}.

In the unweighted method, a center is a set of frequent coordinates selected by frequency thresholding. In the weighted variant, we additionally store a center weight for each selected coordinate. We aggregate within a cluster $\mathcal{S}$ for each coordinate $t$:
\begin{displaymath}
    f_t = |\{ \bm{x}\in\mathcal{S}: x_t=1\}|,
    \qquad
    s_t = \sum_{\bm{x}\in\mathcal{S}} v_t,
\end{displaymath}
and keep coordinates whose aggregate weight $s_t$ is large relative to the maximum:
\begin{displaymath}
    \mathcal{I}(\mathcal{S})
    \;=\;
    \bigl\{t:\; s_t \ge \alpha \cdot s_{\max}\bigr\}, \quad s_{\max}=\max_t s_t.
\end{displaymath}
For each retained coordinate, we store the average contribution
\begin{displaymath}
    \bar{\omega}_t \;=\; \frac{s_t}{\max(1,f_t)},
\end{displaymath}
and represent the center as a sparse weighted vector over $\mathcal{I}(\mathcal{S})$. This yields \textbf{weighted FreqItems} that remain sparse while aligning the center representation with the weighted assignment distance.

\paragraph{Summary} Our weighted $k$-FreqItems extension differs from the original $k$-FreqItems in: 
(i) replacing Jaccard distance $d_J$ with weighted Jaccard distance $d_w$ for assignments;
(ii) replacing MinHash with CWS-based Weighted MinHash in SILK's two-level LSH seeding; and
(iii) updating centers using weighted aggregation.
These extensions preserve the scalability benefits of the original pipeline, and make the backbone compatible with WISE's learned feature weightings.

\subsection{Dataset Details and Preprocessing}
\label{app:impl:datasets}

\begin{table*}[t]
\centering
\small
\renewcommand{\arraystretch}{1.2}
\caption{The statistics of six real-world mixed-type tabular datasets.}
\label{tab:data_sets_stats}
\begin{tabular}{lrrrrlr}
\toprule
\rowcolor[HTML]{D9EAD3}
\textbf{Dataset} & \textbf{\# Samples} & \textbf{\# Dimensions} & \textbf{\# Num. Features} & \textbf{\# Cat. Features} & \textbf{Ground-Truth Feature} & \textbf{\# Classes} \\
\midrule
\textbf{Adult}    & 45{,}222  & 14 & 6  & 8  & Relationship         & 6 \\
\textbf{Vermont}  & 129{,}816 & 48 & 8  & 40 & INCWAGE              & 4 \\
\textbf{Arizona}  & 203{,}353 & 50 & 8  & 42 & INCWAGE              & 4 \\
\textbf{Obesity}  & 2{,}111   & 16 & 8  & 8  & Obesity Level        & 7 \\
\textbf{Credit}   & 690       & 15 & 6  & 9  & A16-class attribute  & 2 \\ 
\textbf{GeoNames} & 500{,}000 & 11 & 5  & 6  & Feature Class        & 10 \\
\bottomrule
\end{tabular}
\end{table*}


We evaluate \textbf{WISE} on six real-world mixed-type tabular datasets spanning demography, healthcare, finance, and geography. 
Table~\ref{tab:data_sets_stats} summarizes their dataset sizes and numerical/categorical feature compositions. 
Unless otherwise specified, ground-truth labels are used \emph{only} for ex post evaluation; all clustering methods operate in a fully unsupervised manner.
\begin{itemize}[nosep, left=5pt]
  \item \textbf{Adult:}
  We use the UCI Adult (Census Income) dataset~\cite{uci_adult_1996}. 
  While the original task is binary income prediction, we derive a multi-class ground truth by using the \texttt{relationship} attribute, which contains six categories. 
  Records with missing values are removed, yielding the final sample size reported in Table~\ref{tab:data_sets_stats}.

  \item \textbf{Vermont and Arizona:}
  These datasets are drawn from Match~3 of the 2018 Differential Privacy Synthetic Data Challenge, corresponding to US Census Bureau PUMS files for Vermont and Arizona \cite{urban_dp_vermont_2020, urban_dp_arizona_2020}. 
  We use the wage attribute as ground truth and discretize it into four classes:
  0 ($\text{wage} = 0$), 1 ($0 < \text{wage} \leq 500$), 2 ($500 < \text{wage} \leq 1{,}000$), and 3 ($\text{wage} >1{,}000$).

  \item \textbf{Obesity:}
  We use the Obesity Levels dataset introduced by Palechor and de la Hoz Manotas~\cite{palechor2019obesity}, which provides seven obesity-level categories as ground-truth labels.

  \item \textbf{Credit:}
  We use the UCI Credit Approval dataset~\cite{uci_credit_approval_1987}, a mixed-attribute tabular dataset, with the target attribute \texttt{A16} serving as ground truth.
  
  \item \textbf{GeoNames:}
  We use the GeoNames geographical database~\cite{geonames}. 
  The \texttt{feature class} field is used as the ground-truth label, while the finer-grained \texttt{feature code} is removed to avoid label redundancy. To ensure scalability, we uniformly subsample 500{,}000 records from the original dump.
\end{itemize}

\subsection{Evaluation Metrics}
\label{app:impl:metrics}

We evaluate clustering performance using a comprehensive set of six widely adopted metrics that jointly capture external agreement, intrinsic structural quality, and computational efficiency (Table \ref{tab:metrics_summary}).
Such a multi-faceted evaluation is particularly important for mixed-type tabular data, where different metrics highlight complementary aspects of clustering behavior.

\begin{table*}[t]
\centering
\small
\renewcommand{\arraystretch}{1.2}
\caption{Evaluation metrics for clustering quality and efficiency.}
\label{tab:metrics_summary}
\begin{tabular}{lll}
\toprule
\rowcolor[HTML]{D9EAD3}
\textbf{Evaluation Facet} & \textbf{Metric} & \textbf{Computation} \\
\midrule
\multirow{4}{*}{\textbf{External Agreement}} 
   & ARI \cite{hubert1985ARI}      & Computed on non-noise set \\
   & NMI  \cite{vinh2009NMI}       & Computed on non-noise set \\
   & Purity \cite{manning2008ir}   & Computed on non-noise set \\
   & ACC \cite{kuhn1955hungarian}  & Computed on non-noise set \\ 
\textbf{Intrinsic Structural Quality} & SWC \cite{rousseeuw1987silhouette} & Calculated using unweighted Gower's distance \cite{gower1971gower-distance} \\
\textbf{Computational Efficiency}  & Time & Total wall-clock time (Seconds) \\
\bottomrule
\end{tabular}
\end{table*}

\paragraph{External Agreement}
When ground-truth labels are available, we report four standard external metrics to quantify the agreement between predicted clusters and reference partitions.
Specifically, we use 
(1) Adjusted Rand Index (\textbf{ARI}) \cite{hubert1985ARI}, which corrects for chance agreement and penalizes both over- and under-segmentation; 
(2) Normalized Mutual Information (\textbf{NMI}) \cite{vinh2009NMI}, which measures shared information in an information-theoretic manner; 
(3) \textbf{Purity} \cite{manning2008ir}, which reflects the dominance of ground-truth classes within clusters; and 
(4) Clustering Accuracy (\textbf{ACC}). Following standard practice, ACC is computed by optimally permuting cluster labels using the Hungarian algorithm \cite{kuhn1955hungarian} before measuring accuracy. 
Together, these metrics provide complementary perspectives on clustering correctness and label-level consistency.

\paragraph{Intrinsic Structural Quality}
To evaluate clustering structure independently of ground-truth labels, we additionally report the Silhouette Coefficient (\textbf{SWC}) \cite{rousseeuw1987silhouette}, which jointly captures intra-cluster cohesion and inter-cluster separation.
As our datasets consist of heterogeneous numerical and categorical attributes, SWC is computed using pairwise Gower distances \cite{gower1971gower-distance}, enabling a principled and unified treatment of mixed feature types and avoiding distortions introduced by purely Euclidean or categorical distances.

\paragraph{Computational Efficiency}
Beyond clustering quality, we assess computational efficiency using end-to-end wall-clock running time (\textbf{Time}, in seconds). 
All methods are evaluated on the same hardware and software stack, and reported times include the complete clustering pipeline--covering preprocessing, representation learning (if applicable), and clustering, to ensure fair comparisons between classical and learning-based approaches and reflect practical deployment costs.

\subsection{Hyperparameter Settings}
\label{app:impl:params}

For all baselines, we use the hyperparameters and training protocols recommended in their original papers and/or official implementations \cite{huang1997kprototypes, svirsky2024IDC, rauf2025TableDC, peng2022tell, somepalli2021saint}. For baselines that cannot directly consume mixed-type inputs, we apply the preprocessing required by their original designs. In particular, for IDC, we evaluate both one-hot encoding and target encoding; for TableDC, we use its auto-encoding pipeline before clustering. For methods that require the number of clusters $K$ as an input but do not infer $K$ automatically, we adopt a controlled cluster-recovery protocol. Specifically, for each dataset with reference class count $K^\ast$, all such methods are evaluated on the same pre-specified candidate grid of $K$ values that covers $K^\ast$, and we report the best-performing setting within that shared grid. This protocol uses external information and should therefore be interpreted as assessing how well each method recovers the benchmark partition granularity, rather than as fully unconstrained unsupervised model selection. Importantly, we do not choose a separate favorable $K$ range for WISE or for any baseline. Different methods may attain their best performance at different values within the same grid; when a method peaks closer to $K^\ast$, this suggests that its preferred clustering granularity is more consistent with the reference partition. The shared candidate grids are listed in Table~\ref{tab:app:k_grid}.

\begin{table}[t]
\centering
\caption{Shared candidate grids of $K$ used in the controlled cluster-recovery protocol.}
\label{tab:app:k_grid}
\begin{tabular}{ll}
\toprule
\rowcolor[HTML]{D9EAD3}
\textbf{Dataset} & \textbf{Candidate grid of $K$} \\
\midrule
\textbf{Adult}    & $\{4,5,6,7,8\}$ \\
\textbf{Obesity}  & $\{2,\dots,10,15,20\}$ \\
\textbf{Credit}   & $\{1,2,3,4,5,6\}$ \\
\textbf{Vermont}  & $\{2,\dots,10,15,20,25,30\}$ \\
\textbf{Arizona}  & $\{2,\dots,10,15,20,25,30\}$ \\
\textbf{GeoNames} & $\{5,\dots,15,20,25,30\}$ \\
\bottomrule
\end{tabular}
\end{table}
\newcolumntype{P}[1]{>{\raggedright\arraybackslash}p{#1}}

\begin{table}[t]
\centering
\small
\renewcommand{\arraystretch}{1.2}
\setlength{\tabcolsep}{4pt} 
\caption{Dataset-specific hyperparameter settings for \textbf{WISE}. $B$ is the number of bits assigned to each column in BEP. LOFO weight sensing uses $T$ trees (shared defaults: $m{=}3$; Q-D balance $\lambda_{\mathrm{QD}}{=}0.5$; \texttt{max\_depth} $=20$; \texttt{min\_samples\_leaf} $=50$; \texttt{train\_sample\_frac} $=0.1$). Clustering Stage~I runs $k$-FreqItems on BEP with $(k_0,\alpha_0,\beta_0)$; Stage~II refines in record space with $(K,\alpha)$.}
\label{tab:app:wise_params_all}

\begin{tabular}{lccP{0.36\columnwidth}P{0.21\columnwidth}}
\toprule
\rowcolor[HTML]{D9EAD3}
\textbf{Dataset} & \textbf{$B$} & \textbf{$T$} &
\textbf{Stage~I Clustering: $(k_0,\alpha_0,\beta_0)$} &
\textbf{Stage~II Clustering: $(K,\alpha)$} \\
\midrule
\textbf{Adult}    & 32 & 30 & $k_0{=}15,\ \alpha_0{=}0.4,\ \beta_0{=}0.9$ & $K{=}5,\ \alpha{=}0.4$ \\
\textbf{Vermont}  & 64 & 30 & $k_0{=}10,\ \alpha_0{=}0.4,\ \beta_0{=}0.4$ & $K{=}4,\ \alpha{=}0.4$ \\
\textbf{Arizona}  & 64 & 30 & $k_0{=}5,\ \alpha_0{=}0.7,\ \beta_0{=}0.6$  & $K{=}6,\ \alpha{=}0.5$ \\
\textbf{Obesity}  & 32 & 30 & $k_0{=}25,\ \alpha_0{=}0.4,\ \beta_0{=}0.4$ & $K{=}5,\ \alpha{=}0.7$ \\
\textbf{Credit}   & 32 & 10 & $k_0{=}5,\ \alpha_0{=}0.2,\ \beta_0{=}0.1$  & $K{=}2,\ \alpha{=}0.2$ \\
\textbf{GeoNames} & 64 & 30 & $k_0{=}50,\ \alpha_0{=}0.2,\ \beta_0{=}0.2$ & $K{=}9,\ \alpha{=}0.2$ \\
\bottomrule
\end{tabular}
\end{table}

\begin{table}[!t]
\centering
\small
\renewcommand{\arraystretch}{1.2}
\caption{Key hyperparameters for baselines (others use official defaults). Abbreviations: OH = one-hot encoding; TE = target encoding; AE = auto-encoding; ER = entity-resolution mode.}
\label{tab:app:baseline_params}
\begin{tabular}{lll}
\toprule
\rowcolor[HTML]{D9EAD3}
\textbf{Dataset} & \textbf{Method} & \textbf{Key Hyperparameters / Settings} \\
\midrule

\multirow{5}{*}{\textbf{Adult}}
& \textbf{k-Proto} & $k{=}5,\ \gamma{=}0.06$ (better than nominal $k{=}6$) \\
& \textbf{IDC} & epochs$=800$; start\_global\_gates\_training\_on\_epoch$=400$; encoding: TE \\
& \textbf{TableDC} & AE dim$=64$; mode: ER; num\_clusters$=6$ \\
& \textbf{TELL} & num\_clusters$=6$ \\
& \textbf{SAINT} & $k{=}6$ \\

\midrule
\multirow{5}{*}{\textbf{Vermont}}
& \textbf{k-Proto} & $k{=}4,\ \gamma{=}0.107$ \\
& \textbf{IDC} & epochs$=2000$; start\_global\_gates\_training\_on\_epoch$=1000$; encoding: OH \\
& \textbf{TableDC} & AE dim$=64$; mode: ER; num\_clusters$=4$ \\
& \textbf{TELL} & num\_clusters$=4$ \\
& \textbf{SAINT} & $k{=}4$ \\

\midrule
\multirow{5}{*}{\textbf{Arizona}}
& \textbf{k-Proto} & $k{=}5,\ \gamma{=}0.11$ (better than nominal $k{=}4$) \\
& \textbf{IDC} & epochs$=2000$; start\_global\_gates\_training\_on\_epoch$=1000$; encoding: OH \\
& \textbf{TableDC} & AE dim$=64$; mode: ER; num\_clusters$=4$ \\
& \textbf{TELL} & num\_clusters$=4$ \\
& \textbf{SAINT} & $k{=}4$ \\

\midrule
\multirow{5}{*}{\textbf{Obesity}}
& \textbf{k-Proto} & $k{=}5,\ \gamma{=}0.122$ (better than nominal $k{=}7$) \\
& \textbf{IDC} & epochs$=400$; start\_global\_gates\_training\_on\_epoch$=200$; encoding: TE \\
& \textbf{TableDC} & AE dim$=64$; mode: ER; num\_clusters$=7$ \\
& \textbf{TELL} & num\_clusters$=7$ \\
& \textbf{SAINT} & $k{=}7$ \\

\midrule
\multirow{5}{*}{\textbf{Credit}}
& \textbf{k-Proto} & $k{=}2,\ \gamma{=}0.05$ \\
& \textbf{IDC} & epochs$=100$; start\_global\_gates\_training\_on\_epoch$=50$; encoding: OH \\
& \textbf{TableDC} & AE dim$=64$; mode: ER; num\_clusters$=2$ \\
& \textbf{TELL} & num\_clusters$=2$ \\
& \textbf{SAINT} & $k{=}2$ \\

\midrule
\multirow{5}{*}{\textbf{GeoNames}}
& \textbf{k-Proto} & $k{=}10,\ \gamma{=}0.065$ \\
& \textbf{IDC} & epochs$=2000$; start\_global\_gates\_training\_on\_epoch$=1000$; encoding: OH \\
& \textbf{TableDC} & AE dim$=64$; mode: ER; num\_clusters$=10$ \\
& \textbf{TELL} & num\_clusters$=10$ \\
& \textbf{SAINT} & $k{=}10$ \\

\bottomrule
\end{tabular}
\end{table}

\paragraph{WISE}
WISE consists of (i) mixed-type conversion via BEP, (ii) LOFO-based weight sensing, (iii) two-stage weight-aware clustering, and (iv) DFI-based interpretation.
In our implementation, most hyperparameters follow our released code defaults (or the defaults of the underlying libraries), and we report the key representation, sensing, and clustering hyperparameters that affect performance.
Concretely, BEP maps each original feature column to a length-$B$ binary block, where $B$ controls the per-column encoding resolution.
For LOFO-based weight sensing (Module~2), we train Random Forest predictors with $T$ trees and select $m$ trees under a quality--diversity (Q--D) objective.
For clustering (Module~3), Stage~I runs repeated $k$-FreqItems on BEP with $(k_0,\alpha_0,\beta_0)$, and Stage~II refines clusters in record space with $(K,\alpha)$. 
Table~\ref{tab:app:wise_params_all} summarizes the dataset-specific settings and the shared defaults used across datasets.
For Stage~I, we follow the tuning guidance of the original $k$-FreqItems procedure \cite{huang2023kfreqitems}: we run $k$-FreqItems on BEP representations and select $(k_0,\alpha,\beta)$ by minimizing MSE under Jaccard distance, where $k_0$ is chosen via an elbow-style sweep.
In our setting, $k_0$ is typically set larger than the final $K$, since Stage~II can automatically merge overlapping Stage~I clusters based on co-occurrence in the record embedding.
Stage~II parameters $(K,\alpha)$ are tuned analogously on the one-hot record embedding of Stage~I assignments, again using MSE under Jaccard distance as the objective.

\paragraph{Baselines}
Table~\ref{tab:app:baseline_params} reports the key hyperparameters that materially affect performance for each baseline; all other hyperparameters are left as default in the official code/recommended setup.
For IDC, we additionally report the encoding choice (one-hot vs.\ target encoding).
For TableDC, we report the embedding dimension and the mode used for converting mixed-type inputs.
When a nearby $k$ yields better results than the nominal $K$, we use that value and note it in the table.

\subsection{Experiment Environment}
\label{app:impl:environment}

All experiments were conducted on a high-performance workstation equipped with an Intel\textsuperscript{\textregistered} Xeon\textsuperscript{\textregistered} Platinum 8480C @ 3.80GHz, 256~GB of RAM, and an NVIDIA H200 GPU to ensure consistent and comparable results. 

\section{Pseudocode}
\label{app:pseudo}

In this section, we provide high-level pseudocode for the \textbf{WISE} framework in Algorithm~\ref{alg:wise_pipeline}.

\begin{algorithm}[h]
\caption{\textsc{WISE}: Weight-Informed Self-Explaining Clustering for Mixed-Type Tabular Data}
\label{alg:wise_pipeline}
\begin{algorithmic}[1]
\REQUIRE 
Dataset $X \in \mathbb{R}^{n\times d}$ with domain info; 
BEP bits $B$;
LOFO params $(T,m,\lambda_{\mathrm{QD}})$;
Stage~I clustering param $k_0$; 
Stage~II clustering param $K$.
\ENSURE 
Final labels $y \in \{0,\dots,K-1\}^n$;
explanations $W_{\mathrm{cluster}}, W_{\mathrm{instance}}$.

\STATE \textbf{(Module 1)} $X_{\mathrm{BEP}} \gets \mathrm{BEP}(X;B)$
\hfill \COMMENT{$X_{\mathrm{BEP}}\in\{0,1\}^{n\times p}$; feature $j$ maps to a fixed bit-group}

\STATE \textbf{(Module 2)} $\{\bm{w}^{(r)}\}_{r=1}^{R} \gets \textsc{LOFO-Sensing}(X;T,m,\lambda_{\mathrm{QD}})$
\hfill \COMMENT{$\bm{w}^{(r)}\in\mathbb{R}^d_{\ge 0},~\|\bm{w}^{(r)}\|_1=1$}

\STATE \textbf{(Module 3) Stage~I Clustering}
\STATE Initialize record matrix $L \in \{0,\dots,k_0-1\}^{n\times R}$

\FOR{$r=1$ \TO $R$}
  \STATE Lift feature-weights to bit-weights: for each bit $b$ from feature $j$, set $v^{(r)}_b \gets w^{(r)}_j$
  \STATE $\ell_r(\cdot) \gets \textsc{Weighted-}k\textsc{-FreqItems}(X_{\mathrm{BEP}}, \bm{v}^{(r)}, k_0)$
  \FOR{$i=1$ \TO $n$}
    \STATE $L[i,r] \gets \ell_r(i)$
  \ENDFOR
\ENDFOR

\STATE \textbf{Stage~II Clustering}
\STATE $Z \gets \textsc{OneHot}(L; k_0)$ \hfill \COMMENT{$Z\in\{0,1\}^{n\times (Rk_0)}$}
\STATE $y(\cdot) \gets k\textsc{-FreqItems}(Z, K)$

\STATE \textbf{(Module 4)} $(W_{\mathrm{cluster}}, W_{\mathrm{instance}}, \mathrm{DFI})
\gets \textsc{DFI-Explain}(L, y, \{\bm{w}^{(r)}\}_{r=1}^{R})$

\STATE \textbf{RETURN} $(y, W_{\mathrm{cluster}}, W_{\mathrm{instance}})$
\end{algorithmic}
\end{algorithm}

\begin{figure*}[!t]
  \centering
  \begin{subfigure}[t]{0.99\textwidth}
    \centering
    \includegraphics[width=\linewidth]{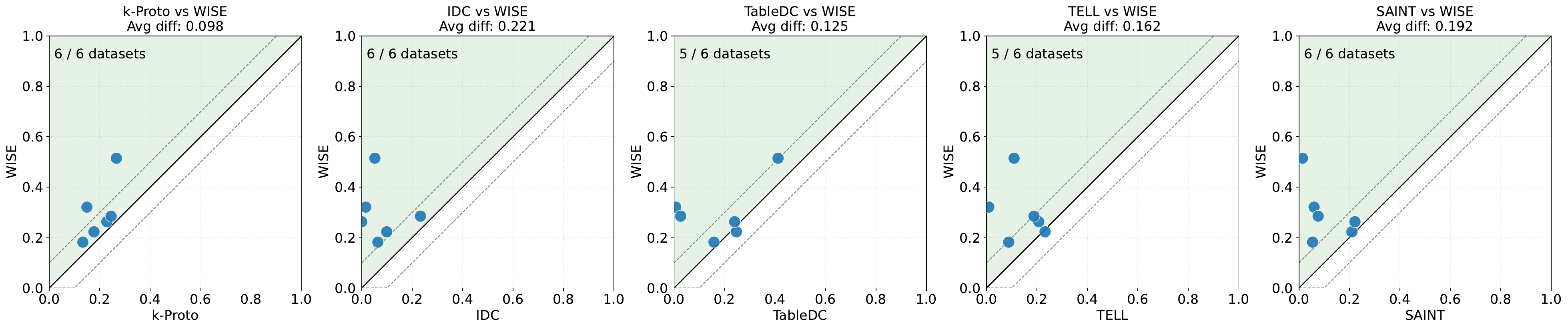}
    \caption{NMI}
    \label{fig:pairwise_nmi}
  \end{subfigure}
  \begin{subfigure}[t]{0.99\textwidth}
    \centering
    \includegraphics[width=\linewidth]{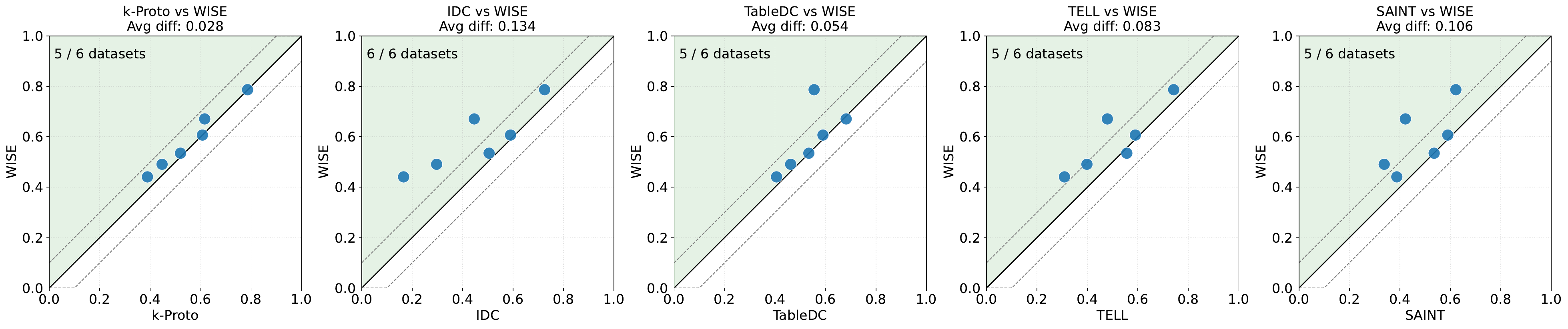}
    \caption{Purity}
    \label{fig:pairwise_purity}
  \end{subfigure}
  \begin{subfigure}[t]{0.99\textwidth}
    \centering
    \includegraphics[width=\linewidth]{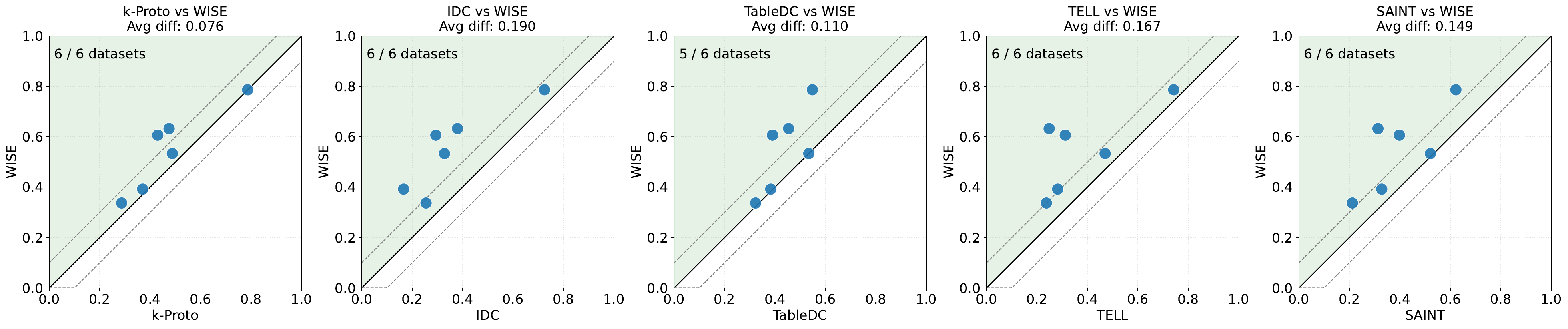}
    \caption{ACC}
    \label{fig:pairwise_acc}
  \end{subfigure}
  \begin{subfigure}[t]{0.99\textwidth}
    \centering
    \includegraphics[width=\linewidth]{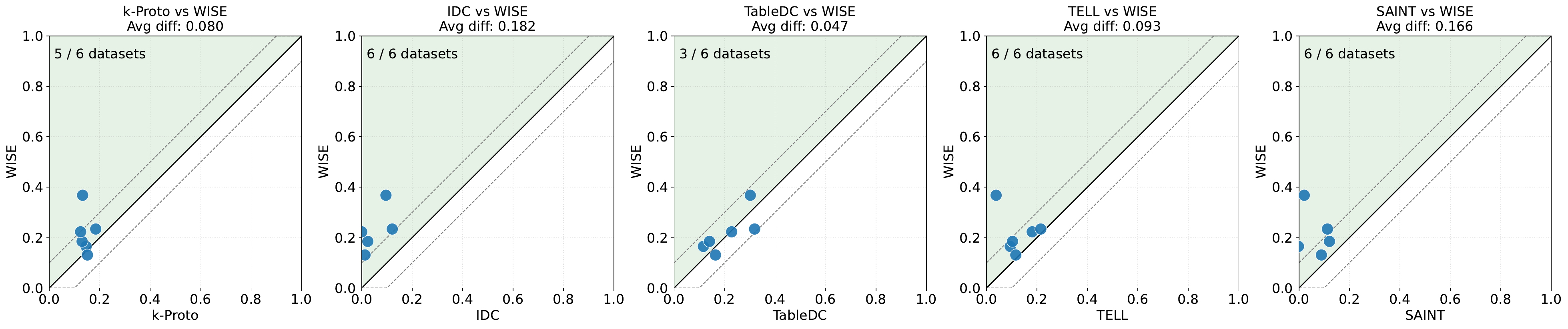}
    \caption{SWC}
    \label{fig:pairwise_swc}
  \end{subfigure}
  \caption{Pairwise comparisons between \textbf{WISE} and each baseline under remaining four metrics.}
  \label{fig:pairwise_all}
\end{figure*}

\section{Additional Results}
\label{app:expt}

\subsection{Additional Comparisons on Main Experiments}
\label{app:expt:main}

In this subsection, we further report pairwise comparisons between \textbf{WISE} and each baseline for NMI, Purity, ACC, and SWC. Each panel summarizes, across the six datasets, how often \textbf{WISE} outperforms a given baseline.

\paragraph{External Metrics.}
The overall picture is consistent with the main results: \textbf{WISE} consistently dominates the baselines on label-based criteria.
For NMI, \textbf{WISE} wins against k-Prototypes, IDC, and SAINT on all datasets, and wins on 5/6 datasets against TableDC and TELL, with clear average improvements.
For Purity, \textbf{WISE} again improves over IDC on all datasets, and achieves 5/6 wins against the remaining baselines, with smaller average margins, as expected for Purity, which can saturate when clusters are already label-dominant.
For ACC, \textbf{WISE} is uniformly strong: it wins on 6/6 datasets against k-Prototypes, IDC, TELL, and SAINT, and on 5/6 against TableDC, with the largest average gaps again appearing versus the deep baselines.

\paragraph{Intrinsic Metric (SWC).}
For SWC, \textbf{WISE} consistently outperforms IDC, TELL, and SAINT, and remains favorable to k-Prototypes (5/6).
TableDC is notably competitive on SWC: \textbf{WISE} wins on 3/6 datasets, ties on one, and trails on 2/6, with a modest average gap of $0.047$. This pattern is expected since SWC is a purely structure-driven criterion that directly rewards within-cluster compactness and between-cluster separation. Despite optimizing a different objective, \textbf{WISE} remains broadly comparable and still achieves a slight overall advantage in SWC.

\subsection{Quantitative Evaluation of Explanation Faithfulness}
\label{app:expt:faithfulness}

\begin{table}[t]
\centering
\caption{Quantitative evaluation of DFI faithfulness.}
\label{tab:dfi_faithfulness}
\begin{tabular}{lccc}
\toprule
\rowcolor[HTML]{D9EAD3}
\textbf{Feature subset} & \textbf{\#Features} & \textbf{Accuracy} & \textbf{Macro-F1} \\
\midrule
All features & 14 & \textbf{0.988} & \textbf{0.986} \\
\midrule
Top-3 DFI & 3 & \underline{0.984} & \underline{0.980} \\
Top-5 DFI & 5 & \textbf{0.988} & \textbf{0.986} \\
Top-10 DFI & 10 & \textbf{0.988} & \textbf{0.986} \\
\midrule
Random-3 & 3 & 0.536 & 0.300 \\
Random-5 & 5 & 0.660 & 0.494 \\
Random-10 & 10 & 0.873 & 0.814 \\
\bottomrule
\end{tabular}
\end{table}

We further quantify the faithfulness of DFI by testing whether the features identified as important by DFI are indeed sufficient to recover the cluster assignments produced by WISE. Concretely, we first compute a global DFI ranking by taking a cluster-size-weighted average of the cluster-level DFI scores. We then train a shallow decision tree to predict the final \textbf{WISE} cluster labels using only the selected features.

Table~\ref{tab:dfi_faithfulness} reports the results on Adult. The top-3 DFI features already retain nearly all predictive power, and the top-5 DFI features exactly match the performance of using all 14 features. In contrast, random feature subsets are substantially weaker, even when using 10 features. The random baselines are averaged over 30 trials. These results indicate that DFI successfully identifies a small subset of features that captures the key decision structure underlying the final clustering.

\end{document}